\documentclass[letterpaper]{article} 
\usepackage{aaai23}  
\usepackage{times}  
\usepackage{helvet}  
\usepackage{courier}  
\usepackage[hyphens]{url}  
\usepackage{graphicx} 
\usepackage{natbib}  
\usepackage{caption} 
\nocopyright

\usepackage{algorithm}
\usepackage{algorithmic}

\usepackage{comment}
\usepackage{lineno}

\usepackage{textcomp}

\usepackage{wrapfig}
\usepackage{xr-hyper}

\usepackage[utf8]{inputenc} 
\usepackage[T1]{fontenc}    
\usepackage{url}            
\usepackage{booktabs}       
\usepackage{amsfonts}       
\usepackage{nicefrac}       
\usepackage{microtype}      

\usepackage{comment}

\usepackage{amsmath,amsfonts,amssymb,mathtools,url,enumerate,verbatim}
\usepackage{amsthm}
\usepackage[german,french,british]{babel}
\usepackage{ucs}
\usepackage{times}
\usepackage{pgf,tikz}
\usepackage{float}
\usepackage{mathrsfs}
\usepackage{mathtools}
\usepackage{multirow}
\usepackage{yfonts}
\usepackage{caption}
\usepackage{subcaption}
\usetikzlibrary{arrows}

    \newcommand\newdot{{\kern.8pt\cdot\kern.8pt}}
\newcommand\nbull{{\kern.8pt\raise1.5pt\hbox{\small\bf .}\kern.8pt}}
\newcommand\1{\hbox{\kern.375em\vrule height1.57ex depth-.1ex
		width.05em\kern-.375em \rm 1}}

\newcommand\E{\mathbb{E}}

\newcommand\N{\mathbb{N}}\newcommand\R{\mathbb{R}}
\renewcommand\P{\mathbb{P}}

\DeclareMathOperator*{\argmin}{arg\,min}

\DeclareMathOperator{\lip}{L_{\ell}}

\DeclareMathOperator{\spec}{\mathcal{I},\sigma}

\DeclareMathOperator{\nuc}{\mathcal{I},\ast}

\DeclareMathOperator{\du}{U}
\DeclareMathOperator{\dut}{U^{\top}}

\DeclareMathOperator{\rad}{\mathfrak{R}}

\DeclareMathOperator{\Fr}{Fr}

\DeclareMathOperator{\Var}{\mathbb{V}ar}

\DeclareMathOperator{\sdv}{\sigma}

\usepackage{mathabx}
\newtheorem{theorem}{Theorem}
\newtheorem{lemma}[theorem]{Lemma}
\newtheorem{proposition}[theorem]{Proposition}
\newtheorem{corollary}[theorem]{Corollary}
\theoremstyle{definition}

\newtheorem{assumption}{Assumption}

\theoremstyle{definition}

\usepackage{newfloat}
\usepackage{listings}
\DeclareCaptionStyle{ruled}{labelfont=normalfont,labelsep=colon,strut=off} 
\floatstyle{ruled}
\newfloat{listing}{tb}{lst}{}
\floatname{listing}{Listing}
\title{Generalization Bounds for Inductive Matrix Completion in Low-noise Settings\footnote{Accepted for presentation at AAAI 2023}}
\author{
    Antoine Ledent\textsuperscript{\rm 1}\thanks{Corresponding author}, 
   Rodrigo Alves\textsuperscript{\rm 2},
 Yunwen Lei\textsuperscript{\rm 3},
 Yann Guermeur\textsuperscript{\rm 4}, and
 Marius Kloft\textsuperscript{\rm 5}
}
\affiliations{
    \textsuperscript{\rm 1} Singapore Management University (SMU)\\ aledent@smu.edu.sg\\
      \textsuperscript{\rm 2} Czech Technical University in Prague (CTU)\\ rodrigo.alves@fit.cvut.cz\\
       \textsuperscript{\rm 3} Hong Kong Baptist University (HKBU)\\ yunwen.lei@hotmail.com\\
        \textsuperscript{\rm 4} Centre National de la Recherche Scientique (CNRS)\\ yann.guermeur@loria.fr\\
         \textsuperscript{\rm 5} Technische Universit\"at Kaiserslautern (TUK)\\
         kloft@cs.uni-kl.de
       
}

\DeclareMathOperator{\dualnorm}{\text{d}}

\DeclareMathOperator{\loss}{\ell}
\DeclareMathOperator{\losb}{B_{\ell}}

\usepackage{lineno}

\newcommand*{\red}{\textcolor{red}}

\begin{document}


\maketitle

\begin{abstract}
	
	We study inductive matrix completion (matrix completion with side information) under an i.i.d. subgaussian noise assumption at a low noise regime, with uniform sampling of the entries. We obtain for the first time generalization bounds with the following three properties: 
	(1) they scale like the standard deviation of the noise and in particular approach zero in the exact recovery case; (2) even in the presence of noise, they converge to zero when the sample size approaches infinity; and (3) for a fixed dimension of the side information, they only have a logarithmic dependence on the size of the matrix. Differently from many works in approximate recovery, we present results both for bounded Lipschitz losses and for the absolute loss, with the latter relying on Talagrand-type inequalities.  The proofs create a bridge between two approaches to the theoretical analysis of matrix completion, since they consist in a combination of techniques from both the exact recovery literature and the approximate recovery literature.
\end{abstract}

\section{Introduction}

Matrix Completion (MC), the problem which consists in predicting the unseen entries of a matrix based on a small number of observations, presents the rare combination of (1) a rich mathematical playground rife with fundamental unsolved problems, and (2) a wealth of unexpected applications in lucrative  and meaningful fields, from Recommender Systems~\cite{optimization4,FakeMKL+,book} to the prediction of drug interaction~\cite{IMCforDrug}. 

One of the most celebrated algorithms for standard matrix completion is the Softimpute algorithm~\cite{softimpute}, which solves the following optimization problem:
\begin{align}
\label{MC}
\min_{Z\in\R^{m\times n}} \frac{1}{2} \|P_{\Omega}(Z-R)\|_{\Fr}^2 +\lambda \|Z\|_{*},
\end{align} 
where $P_{\Omega}$ denotes the projection on the set $\Omega$ of observed entries, $R$ is the ground truth matrix,  $\|\nbull\|_{*}$ denotes the \textit{nuclear norm} (the sum of the matrix's singular values) and $\|\nbull\|_{\Fr}$ denotes the Frobenius norm. The idea of the algorithm is to encourage \textit{low-rank} solutions in a similar way to how $L^1$ regularization encourages component sparsity. The parameter $\lambda$ must be tuned with cross-validation.

\textit{Inductive matrix completion} (IMC)~\cite{IMC1,IMC,IMC2,IMC3} is another closely related model which assumes that additional information is available in the form of feature vectors for each user (row) and item (column). It assumes that the side information is summarized in matrices $X\in \R^{m\times a}$ and $Y\in \R^{n\times b}$. IMC then optimizes the following objective function
\begin{align}
\label{IMC}
\min_{M\in\R^{a\times b}} \frac{1}{N} \|P_{\Omega}(XMY^\top-R)\|_{\Fr}^2+\lambda \|M\|_{*}.
\end{align} 

An interesting question is whether one can provide sample complexity guarantees for the optimization problem above. Typically, doing so requires minor modification to the problem for technical convenience.
There are several such analogues optimization problems~\eqref{MC} and~\eqref{IMC}, depending on the type of statistical guarantee expected and the assumptions: in exact recovery (with the assumption of perfectly noiseless observations), the Frobenius norm is replaced by a hard equality constraint, whilst in approximate (noisy) recovery, the nuclear norm regrulariser is replaced by a hard contraint. 

More precisely, \textit{exact recovery} results study the following hard version of the optimization problem: 
\begin{align}
\label{MChard}
&\min_{Z\in\R^{m\times n}}\| Z\|_{*} \quad \text{subject to}\nonumber\\
& Z_{i,j}=R_{i,j} \quad \forall (i,j)\in\Omega.
\end{align}

In the case of IMC, the equivalent hard version is:
\begin{align}
\label{IMChard}
&\min_{M\in\R^{a\times b}}\| M\|_{*} \quad \text{subject to}\nonumber \\
& [XMY^\top]_{i,j}=R_{i,j} \quad \forall (i,j)\in\Omega.
\end{align}
The study of problem~\eqref{MChard} is the earliest branch of the related literature: it was shown in a series of papers (\citealp{Genius,CandesRecht,SimplerMC}, to name but a few) that if the number of samples is $\geq \widetilde{O}(nr)$  (where $r$ is the rank and $n$ is the size of the matrix, i.e. the number of rows or columns, which ever is larger), then it is possible to recover the whole matrix exactly with high probability as long as the entries are sampled uniformly at random. There has also been some more recent interest in the problem~\eqref{IMChard}: it was shown in~\cite{IMCtheory1} that \textit{assuming the side information $X,Y$ is made up of orthonormal columns}, exact recovery is possible as long as the number of samples $N=|\Omega|$ satisfies $\widetilde{O}(ar)\leq N\leq \widetilde{O}(abr)$. Here, the $\widetilde{O}$ notation hides logarithmic factors in all relevant quantities (including the size $m\times n$ of the matrix).

\textit{Approximate recovery} results typically study modified problems such as the problem below, for which Equation~\eqref{MC} can be interpreted as a Lagranian form):
\begin{align}
\label{MCsoft}
&\min_Z \frac{1}{|\Omega|}\sum_{(i,j)\in\Omega} \ell(R_{i,j},Z_{i,j})\quad \text{subject to}\nonumber \\
& \quad \quad \quad  \|Z\|_*\leq \mathcal{M},
\end{align}
for some loss function $\ell$ which is typically assumed to be bounded and Lipschitz, and some constant $\mathcal{M}$ which must be tuned through cross-validation in a way analogous to the tuning of $\lambda$ in equation~\eqref{MC} in real-life applications. 
In the case of IMC, the equivalent problem is:
\begin{align}
\label{IMCsoft}
&\min_M \frac{1}{|\Omega|}\sum_{(i,j)\in\Omega} \ell([XMY^\top]_{i,j},Z_{i,j})\quad \text{subject to}\nonumber \\
&  \quad \quad \quad \|M\|_*\leq \mathcal{M}.
\end{align}
Approaching the problem this way allows one to deploy the machinery of Rademacher complexities from traditional statistical learning theory to obtain uniform bounds on the generalization gap of any predictor in the given class. Using such techniques, bounds of $\widetilde{O}\left(\sqrt{\frac{nr}{N}}\right)$ (resp. $\widetilde{O}(a^2 r/\sqrt{N})$, more recently $\widetilde{O}\left(\sqrt{\frac{ar}{N}}\right)$) were shown for approximate recovery MC (resp. IMC) under uniform sampling (MC: see~\cite{ReallyUniform1,ReallyUniform2}, IMC, see~\cite{mostrelated,LedentIMC}, cf. also related works). In the distribution-free case, the corresponding rates are  $\widetilde{O}\left(\sqrt{\frac{n^{3/2}r^{1/2}}{N}}\right)$  and $\widetilde{O}\left(\sqrt{\frac{a^{3/2}r^{1/2}}{N}}\right)$. 

The above rates do not make any assumptions on the noise whatsoever, and depend only on explicit dimensional quantities: they are classified as "uniform convergence" bounds in the classic paradigm of statistical learning theory. In particular, while they do also apply to the noiseless case, they are subsumed by the exact recovery results in this case provided the exact recovery threshold is reached.

Thus the most striking hole in the existing theory is the chasm between exact recovery and approximate recovery in Inductive Matrix Completion: on the one hand, we know that if the entries are observed exactly, solving problem~\eqref{IMChard} will eventually recover the whole matrix exactly with high probability given enough entries. On the other hand, we know from the approximate recovery literature that regardless of the noise distribution, solving a properly cross-validated version of problem~\eqref{IMCsoft} will allow us to approach the Bayes error at speed at least $1/\sqrt{N}$ as we observe more entries. It seems reasonable to expect that in real life, neither of these approaches fully explains the statistical generalization landscape of the problem:  we never expect to  observe the entries exactly, and the ground truth is probably not exactly low-rank either, but we still do not expect convergence to the Bayes error to be as slow as in the worst case. What would be more reasonable to expect is a sharp decline of the error around a threshold value before which no method can work even if the entries are observed exactly, followed by a slower decline as the model refines its predictions and evens out the noise in the observations.  This can be observed practically as well, as can be seen  from Figure~\ref{omegaline} in the experiments section: the decay of the error as the number of samples increases is neither convex  (unlike the functions $1/\sqrt{N}$ and $1/N$), nor completely abrupt (as exact recovery results suggest), which indicates the presence of a threshold phenomenon. 

 In this paper, we theoretically capture this phenomenon through generalization error bounds for the solutions to problem~\eqref{IMC}  when the ground truth matrix is observed with some subgaussian noise of subgaussianity constant $\sdv$. In addition, our results completely  remove the orthogonality assumptions on the side information matrices $X,Y$ which are present in the related work~\cite{IMCtheory1}, thus improving the state of the art even in the exact recovery case. 
 
 In summary, we make the following important contributions: 
\begin{enumerate}
	\item We prove (cf. Theorem~\ref{thm:exactOnotation}) that \textit{exact recovery} is possible for IMC  (when the entries are observed exactly) with probability $1-\Delta$ given 
	$\widetilde{O}\left(\mu^5 r^2(a+b)\sigma_0^{-4}  \log\left(\frac{mn}{\Delta}\right) \right )$	samples or more. This is a significant extension of the results in~\cite{IMCtheory1} in that we remove most many of their assumptions. In the formula above, $\mu$ is a measure of incoherence, and $\sigma_0$ denotes the smallest singular value of $X$ or $Y$ assuming they are normalized so that the largest singular value is $1$ in each case. This means that after suitable scaling, $\sigma_0$ can be replaced by the ratio between the largest and smallest singular values of either $X$ or $Y$. The presence of this factor underpins one of the main differences between~\cite{IMCtheory1} and our work. Indeed, the most limiting assumption in~\cite{IMCtheory1} is that the columns of the side information matrices $X$ and $Y$ are \textit{orthonormal}, which is equivalent to assuming that $\sigma_0=1$. 
	\item We experimentally observe the two-phase phenomenon described above via synthetic data experiments. 
	\item We prove generalization bounds (cf. Theorem~\ref{thm:approxOnotation}) which capture this phenomenon in the case of bounded loss functions such as the truncated $L^2$ loss. Indeed, we show that as long as $N$ exceeds the threshold from the exact recovery result, the expected loss scales as
	$\widetilde{O}\left(\sigma_0^{-2}\sdv \mu \frac{\sqrt{a^3b}}{\sqrt{N}} \log^{3}(N/\Delta)\right)$, where $\sdv$ is the subgaussianity constant of the noise. If $\sigma$ is very small, this implies that before the exact recovery threshold (ERT) is reached, the best available bounds are the uniform convergence bounds (which are vacuous at that regime), whereas as soon as the ERT is crossed, our bounds become valid and already have a small value, which continues to drop further as the number of samples increases. This partially explains the sharp drop in the reconstruction error around the ERT even in the noisy case. 
	\item Using Talagrand-type inequalities, we  further prove a similar result (cf. Theorem 3) which applies to the absolute loss $\ell(x,y)=|x-y|$, despite the fact that it is unbounded. 
\end{enumerate}
Note as a side benefit that both of the last two results apply to the Lagrangian formulation of the IMC problem, unlike most of the existing literature on approximate recovery. 

Our second result creates a bridge between the approximate recovery literature and the exact recovery literature: as the subgaussianity constant of the noise $\sdv$ converges to zero, so does the error: the result then reduces to our exact recovery result. Furthermore, our proof techniques also marry both approaches: we rely \textit{both} on the geometry of dual certificates (the tool of choice in the exact recovery literature) \textit{and} Rademacher complexities to reach our result. Beyond our current preliminary results, we believe that the direction we initiate here will prove fertile and that many improved results can be proved, bringing us closer to a complete understanding of the sample complexity landscape of nuclear norm based Inductive Matrix Completion.

\section{Related work}



\noindent \textbf{Perturbed exact recovery with the nuclear norm:} For Matrix Completion without side information, bounds which capture the two-phase phenomenon by incorporating a multiplicative factor of the the variance of the noise 
 have been shown: in~\cite{noisycandes}, a bound of order $O\left(\sqrt{\frac{n^3}{N}}\sdv +\sdv\right)$  is shown for the $L^2$ generalisation error of matrix completion with noise of variance $\sdv$ (Cf. equation III.3 on page 7). The proof relies on a perturbed version of the exact recovery arguments presented in~\cite{Genius}. The result considers a different loss function and does not consider side information and the proof is purely based on directly computing various norms without relying on Rademacher complexities. In a recent and very impressive contribution~\cite{ChenChi20} provided some bounds in the same setting with a finer multiplicative dependency on the size of the matrix $n$ that matches the order of magnitude of the exact recovery threshold (when expressed in terms of sample complexity). The proof is very involved and contrary to our work, the results do not apply to inductive matrix completion.

\noindent \textbf{Exact recovery with the nuclear norm:} In~\cite{SimplerMC}, extending and simplifying earlier work of~\cite{Genius,CandesRecht}, the author proves that exact recovery is possible for matrix completion with the nuclear norm with $\widetilde{O}(nr)$ entries. The result is extended to the case where side information is present in~\cite{IMCtheory1} where it is shown exact recovery is possible with $\widetilde{O}((a+b)r)$ observations, where $a,b$ are the sizes of the side information. However, the result only applies as long as this side information consists of orthonormal columns, significantly reducing the applicability. Other variations of the results exist with improved dependence on certain parameters such as the incoherence constants~\cite{Chen_2015}. 


\noindent \textbf{Perturbed exact recovery for other algorithms}
in learning settings other than nuclear norm minimization, there  is some work with low-noise regimes where the bounds also approach zero as the noise approaches zero (for large enough $N$).  For instance, some work on max norm regularisation has this property~\cite{maxnorm}. Some results of order $\widetilde{O}\left(\sdv\sqrt{\frac{nr}{N}}\right)$ were also obtained for matrix completion with a special algorithm that \textit{requires explicit rank restriction}~\cite{noisy,wang2021matrix}.


\noindent \textbf{Approximate recovery results:} There is a wide body of works proving uniform-convergence type generalization bounds for various matrix completion settings. the vast majority are of order $\widetilde{O}(1/\sqrt{N})$, with most bounds differing from each other in their dependence on other quantities such as $m,n,r,\mu, \sigma$ and (in IMC) $a,b$. 
For matrix completion, ~\cite{ReallyUniform1,ReallyUniform2} proves bounds of order $\widetilde{O}\left(\sqrt{\frac{n^{3/2}r^{1/2}}{N}}\right)$ in the \textit{distribution-free setting} with replacement, as well as $\widetilde{O}\left( \frac{nr\log(n)}{N} +\sqrt{\frac{\log(1/\delta)}{N}}  \right)$ in the transductive setting (i.e. for \textit{uniform} sampling \textit{without} replacement). In the case of inductive matrix completion, rates of $\widetilde{O}\left(\sqrt{\frac{rab}{N}}\right)$ were shown in~\cite{mostrelated,mostrelatedearly,espain} in a distribution-free situation, whilst~\cite{LedentIMC} provides rates of order $\widetilde{O}\left(\sqrt{\frac{ra}{N}}\right)$  and $\widetilde{O}\left(\sqrt{\frac{r^{1/2}a^{3/2}}{N}}\right)$  in the uniform sampling and distribution-free cases respectively. Similar rates were implicitly proved in the more algorithmic contribution~\cite{omic} under very strict assumptions on the side information $X,Y$.   
It is also worth noting that although the component of our result  which involves the subgaussianity of the noise is vacuous when the size of the side information approaches that of the matrix, that is also the case of every approximate recovery result for IMC to date except the very recent paper~\cite{LedentIMC}, whose results are also uniform convergence bounds. Our bounds are far tighter those in all of those works when the noise is small. 

\noindent \textbf{Matrix sensing:}
Matrix sensing is a learning setting with some similarities to inductive matrix completion where rank-one measurements $\langle vw^\top, R \rangle$ of an unknown matrix $R$ are taken, and the matrix $R$ is estimated. There are a wide variety of results depending on the assumptions on the matrix and the sampling distribution~\cite{physics,Aachen,vary,IMCtheory2}. In most cases, the measurements are sampled i.i.d. from some distribution, which introduces some substantial technical differences to the IMC setting. Often, the underlying measurements need to satisfy the restricted isometry property, which is not directly comparable to the joint incoherence assumptions on the side information matrices made in this paper and in the IMC literature. In addition, most results relate to pure exact recovery rather than a low-noise model such as the one studed here. 

\section{Notation and setting}

We assume there is an unknown ground truth matrix $R\in\R^{m\times n}$ that we observe noisily. To draw a sample from the distribution, we first sample an entry $\xi=(\xi_1,\xi_2)=(i,j)$ from the uniform distribution over $[m]\times [n]$. We then observe the quantity $R_{(i,j)}+\zeta_{(i,j)}$ where $\zeta_{(i,j)}$ is the noise, whose distribution can depend on the entry $(i,j)$. The samples are drawn i.i.d.

We suppose we have a training set of $N$ samples and we write $\Omega$ for the set of sampled entries $\xi^1,\xi^2,\ldots,\xi^N$. It is possible to sample the same entry several times (which results in potentially different observations due to the i.i.d. nature of the noise). However, for simplicity of notation we will sometimes write $\sum_{(i,j)\in\Omega} f(R_{(i,j)})$ instead of $\sum_{\xi\in\Omega } f(R_{\xi_{1},\xi_2},\xi)$ as long as no ambiguity is possible. 
We are given two side information matrices $X\in\R^{m \times a}$ and $Y\in\R^{n\times b}$. Throughout this paper, $\|\nbull\|$ denotes the spectral norm, $\|\nbull\|_{\Fr}$ denotes the Frobenius norm, $\|\nbull\|_{*}$ denotes the nuclear norm, and for any integer $l$,  $[l]=\{1,2,\ldots,l\}$.

We make the following assumptions throughout the paper: 
\begin{assumption}[Realizability]
	There exists a matrix $M_*\in \R^{a\times b}$ such that  $R=XM_{*}Y^\top$.
\end{assumption}

\begin{assumption}[Assumptions on the subgaussian noise]
We assume the noise is $\sdv$ subgaussian: $\E(\zeta)=0$ and $\P(|\zeta|\geq t)\leq 2\exp(-t^2/(2\sdv^2))$ for  all $t$.
	
\end{assumption}


We will write $\widebar{X}$ and $\widebar{Y}$ for the matrices obtained by normalizing the columns of $X,Y$ and we will write $\Sigma_1,\Sigma_2$ for the diagonal matrices containing the singular values of $X,Y$. Similarly we will also write $\widebar{\widebar{X}}=\widebar{X}\Sigma_1$ etc. 

We also make the following incoherence assumption. 
\begin{assumption}
There exists a constant $\mu$ such that the following inequalities hold.
    \begin{align}
\|\bar{X}\|_{\infty}\leq \sqrt{\frac{\mu}{m} },\nonumber & \quad \quad 
\|\bar{Y}\|_{\infty}\leq \sqrt{\frac{\mu}{n} },\nonumber \\
\|A\|_{\infty}\leq \sqrt{\frac{\mu}{a}}, &\quad \quad 
\|B\|_{\infty}\leq \sqrt{\frac{\mu}{b}},\label{def:incoherencestrict}
\end{align}
Here the matrices $A,B$ are from the SVD decomposition of the ground truth core matrix $M_*=ADB^\top$ for some diagonal $D$.
\end{assumption}

Note that we do not make the assumption that the matrices $X,Y$ have orthonormal columns (and in particular constant spectrum) as in \cite{IMCtheory1}. Therefore, to cope with such extra difficulty~\eqref{def:incoherencestrict} is needed in the general non orthogonal case. Whilst that reference simply assumes that the column spaces of $X,Y$ are $\mu$ incoherent, our assumption requires that \textit{each individual eigenspace} corresponding to each singular value of $X$, $Y$ and $M$ \textit{be $\mu$-incoherent}.  In the supplementary we explain to what extent this slightly stronger assumption is necessary in the  non-orthogonal case.

\noindent \textbf{Optimization problem:} 
whether considering inductive matrix completion or matrix completion with the nuclear norm, it is common to assume that the entries are sampled exactly (without noise) and that the algorithm used to recover the ground truth is the following:
\begin{align}
\label{eq:optnonoise}
\argmin \left(\|M\|_* \> \text{s.t.} \> \forall (i,j)\in\Omega,  [XMY^\top]_{i,j}=R_{i,j}\right).
\end{align}
This is also the optimization problem we study in the exact recovery portion of our results.

In real situations where there is some noise, some relaxation of the problem is necessary. From an optimization perspective, the most common strategy is to minimize the $L^2$ loss on the observed entries plus a nuclear norm regularisation term: 
\begin{align}
\label{eq:optreallife}
\min \frac{1}{N}\sum_{\xi\in\Omega} \left|[R_{\xi}+\zeta_\xi]-XMY^\top \right|^2+\lambda \|M\|_{*},
\end{align}
where $\lambda$ is a regularization parameter. 
The problem we will consider in this paper is the one defined by equation~\eqref{eq:optreallife}. We will also need to impose the following conditions on $\lambda$:

	\begin{align}
	\label{eq:lambdacondmainpaper}
	\frac{\sdv \sigma_0^2}{C\sqrt{aN}}\leq \lambda \leq \frac{C\sdv \sigma_0^2}{\sqrt{aN}}.
	\end{align}
	for some constant $C$. It is assumed that $\lambda$ has been cross-validated to reach a value which satisfies these conditions.

\section{Main results}

\subsection{Exact recovery}
We have the following extension of the main theorem in~\cite{IMCtheory1}: 
\begin{theorem}
	\label{thm:exactOnotation}
	Assume that the entries are observed without noise and that the strong incoherence assumption~\eqref{def:incoherencestrict} is satisfied for a fixed $\mu$. 
	For any $\Delta>0$ as long as 
	$$N\geq \widetilde{O}\left(\mu^5 r^2(a+b)\sigma_0^{-4}  \log\left(\frac{mn}{\Delta}\right) \right ),$$
	with probability $\geq 1-\Delta$ we have that any solution 
	$M_{\min}$ to the optimization problem below
	\begin{align}
	M_{\min}&\in\argmin  \|M\|_* \quad \text{s.t.} \nonumber  \\  \forall (i,j)\in &\Omega,  \quad  [XMY^\top]_{i,j}=R_{i,j},
	\end{align}
	satisfies 
	$$XM_{\min}Y^\top=R.$$  
	Here, as usual, the $\widetilde{O}$ notation hides further log terms in the quantities $m,n,\sigma_0^{-1},\log(\frac{mn}{\Delta})$.
\end{theorem}
\textbf{Remark:} The above optimization problem can be seen as a limiting case of ~\eqref{eq:optreallife} with $\lambda\rightarrow 0$.\\
\noindent \textbf{Remark:} The above theorem has several advantages over the main theorem in~\cite{IMCtheory1}:
\begin{enumerate}
	\item It is expressed entirely in terms of a fixed high probability $1-\Delta$ (as opposed to relying on dimensional quantities in the expression for the high probability).
	\item It works without assuming that the side information matrices have unit singular values. This is quite a significant improvement as the result in~\cite{IMCtheory1} only holds when the side information matrices belong to a given set of measure zero. There is a quadratic dependence on $\sigma_0^{-1}$ (the inverse of the smallest singular value of either $X$ or $Y$), which matches the dependence in~\cite{PIMC} (although that paper works with a completely different optimization problem away from traditional nuclear norm regularization). 
	\item It holds for any value of $N$, whereas the result in~\cite{IMCtheory1} required $N\leq \widetilde{O}(abr)$ and the result in~\cite{SimplerMC} (which concerns standard MC without side information) required $N\leq mn$.
\end{enumerate}

\subsection{Approximate recovery in a low-noise setting}

Below we present theorems which provide generalization bounds for the IMC model~\eqref{IMC} with the favourable property that they improve when the noise is reduced, and they reduce exactly to the exact recovery result when $\sdv=0$.

The following theorem provides a generalization bound of order $\widetilde{O}\left(a^{3/2}\sqrt{b} \mu \sigma_0^{-2} \sdv  \sqrt{\frac{1}{N}}\right)$ for a bounded Lipschitz loss.

\begin{theorem}
	\label{thm:approxOnotation}
Let $\ell$ be an $\lip$-Lipschitz loss function bounded by $B_\ell$.
	Assume that condition~\eqref{eq:lambdacondmainpaper} on $\lambda$ holds. 
	For any $\Delta>0$, with probability $1-\Delta$ as long as 	$$N\geq \widetilde{O}\left(\mu^5 r^2(a+b)\sigma_0^{-4}  \log\left(\frac{mn}{\Delta}\right) \right ),$$ we have the following bound on the performance of the solution $\widehat{R}$ to the optimization problem~\eqref{IMC}:
	
		\begin{align}
	\label{eq:thefinalmain}
&\mathbb{E}_{(i,j)\sim\mathcal{U}} (\ell(\widehat{R}_{(i,j)},[R+\zeta]_{(i,j)})) \leq   \\
&
	 O\left(a^{3/2}\sqrt{b} \mu \sigma_0^{-2} \sdv \lip  \log^{3}\left(\frac{Nmn}{\Delta}\right) \sqrt{\frac{1}{N}} +B_\ell\frac{\log(\frac{1}{\Delta})}{N}\right),\nonumber 
	\end{align}
where $\mathcal{U}$ stands for the uniform distribution on the entries $[m]\times [n]$.

\end{theorem}

Next, our proof techniques also allow us to prove results which apply to the absolute value loss, despite the fact that it is unbounded.  Indeed, a bound of order $\sqrt{N}$ on the nuclear norm of the difference between the solution and the ground truth is a byproduct of the approximations we perform before applying Rademacher arguments. It can also be used to provide a bound on the \textit{effective} value of $B_\ell$, still yielding an overall rate of $1/\sqrt{N}$ thanks to the fact that the last term in equation~\eqref{eq:thefinalmain} has the strong decay $1/N$. This is a result of our use of the more fine-grained, talagrand-type results from~\cite{localrad} and would not have been possible if we had used standard results on Rademacher complexities such as~\cite{Bartlettmend}.

\begin{theorem}
\label{thm:absoluteloss}
	Assume that condition~\eqref{eq:lambdacondmainpaper} on $\lambda$ holds. 
	For any $\Delta>0$, with probability $1-\Delta$ as long as 	$$N\geq \widetilde{O}\left(\mu^5 r^2(a+b)\sigma_0^{-4}  \log\left(\frac{mn}{\Delta}\right) \right ),$$ we have
		\begin{align}
	\label{eq:thefinalmain}
&\mathbb{E}_{(i,j)\sim \mathcal{U}} \left|\widehat{R}_{(i,j)}-[R+\zeta]_{(i,j)}\right|\leq  \nonumber \\
&
	 O\left(a^{3/2}\sqrt{b} \mu \sigma_0^{-2} \sdv \lip  \log^{3}\left(\frac{Nmn}{\Delta}\right) \sqrt{\frac{1}{N}} \right).
	\end{align}
 Here, $\mathcal{U}$ stands for the uniform distribution on the entries $[m]\times [n]$.
\end{theorem}

\section{Proof strategy}

The main ideas of our proof are (1) to redefine a norm on $\R^{m\times n}$ matrices that captures the effect of  the side information matrices, and (2) to combine proof techniques from both the approximate recovery literature and the exact recovery literature: we perturb the analysis from the exact recovery literature to obtain a bound on the discrepancy between the ground truth and the recovered matrix, and then bootstrap the argument by exploiting the i.i.d. nature of the noise and results from traditional complexity analysis to yield a generalization bound. 

In this informal description, we sometimes write formulae with such as $P_{\Omega}(\widehat{R}-R)$, denoting the projection of $\widehat{R}-R$ onto the set of matrices whose non zero entries are in $\Omega$, which requires assuming that each entry was sampled only once. However, this assumption is made purely for simplicity of exposition and it is not made or needed in the formal proofs in the supplementary.

\subsection{Background on existing techniques}
The main strategy of the proof of the exact recovery results in both~\cite{IMCtheory1} and~\cite{SimplerMC}, which goes back to earlier work~\cite{Genius,CandesRecht,noisycandes} is to use the duality between the nuclear norm and the spectral norm to study the behavior of the nuclear norm around the ground truth. 

It is easiest to explain the strategy in the case of standard matrix completion (as in~\citealp{SimplerMC,noisycandes} etc.). For a given matrix $R$ with singular value decomposition $EDF^\top$, if the columns and rows of $W$ are orthogonal to those of $R$ and it satisfies $\|W\|\leq 1$, the matrix  $\mathcal{Y}:=EF^\top+W$ is a \textit{subgradient to the nuclear norm} at $R$, and a solution to the maximization problem 
\begin{align}
\max_{\mathcal{Y}}  & \left\langle \mathcal{Y} ,R\right\rangle  \quad \text{subject to}\nonumber \\
\|\mathcal{Y}\|&\leq 1.\nonumber  
\end{align}

The subgradients as above allow us to understand the local behavior of the nuclear norm around the ground truth, and one of the most important observations in the early exact recovery analysis of matrix completion is that exact recovery is guaranteed if there exists such a subgradient \textit{whose non zero entries are all in the set of observed entries} and whose spectral norm is $<1$. A subgradient with this property is referred to as a \textit{dual certificate}. Indeed, we have the following result from~\cite{noisycandes}: 
\begin{lemma}
	If there exists a dual certificate $\mathcal{Y}$, then  for any $Z$ with $Z_{i,j}=0\quad \forall (i,j)\in\Omega$  we have \begin{align}
	\|R+Z\|_{*}\geq \|R\|_*-(1-P_{T^\top}(\mathcal{Y}))\|P_{T^\top}(Z)\|_{*}.
	\end{align}
	In particular, $R$ is the unique solution to the optimization problem~\eqref{eq:optnonoise}. Here $P_{T}(Z)=ZP_F+P_EZ-P_EZP_F$ where $P_E$ and $P_F$ are the projection operators onto the column and row spaces of the ground truth respectively. 
\end{lemma}

The high-level intuition behind such a result is that if the set of "observable" matrices whose entries are constrained to lie in the set of observed entries is big enough to contain suitable subgradients, then it is big enough to make the solution to~\eqref{eq:optnonoise} unique.

Whilst most of the early works in the field~\cite{Genius,CandesRecht} work with sampling without replacement and rely on complex combinatorial arguments to prove the existence of a dual certificate, the breakthrough in the work of~\cite{SimplerMC} is to sample with replacement (simplifying the concentration arguments) and to show that the existence of an \textit{approximate} dual certificate is also enough to guarantee uniqueness. More precisely, let $Z\in\R^{\Omega^{\top}}$ be a matrix with zeros in all entries outside $\Omega$, and let $U,U^\top$ be the canonical subgradients of $R$ and $P_{T}(Z)$ respectively. Assume there is an approximate dual certificate $\mathcal{Y}$ with the property that $\| U-P_T(\mathcal{Y})\|_{\Fr}$ is very small and $P_{T^\top}(\mathcal{Y})<1/2$, then we have 
\begin{align}
\label{eq:recht}
&\|R+Z\|_*\nonumber \\
&\geq \left\langle  U+U^\top     ,   R+Z           \right\rangle  \nonumber \\
&= \|R\|_*+ \left\langle  U+U^\top     ,   Z           \right\rangle \nonumber \\
&=\|R\|_*+ \left\langle  U-P_T(\mathcal{Y})    ,   P_T(Z)           \right\rangle\nonumber \\& \quad \quad +\left\langle  U^\top -P_{T^\top}(\mathcal{Y})    ,   P_{T^\top}(Z)           \right\rangle  \nonumber \\
&\geq \|R\|_{*}-\|U-P_T(\mathcal{Y}) \|_{\Fr} \| P_T(Z)    \|_{\Fr}   \nonumber \\& \quad \quad +\|P_{T^\top}(Z)\|_* \left(   1-\|P_T(\mathcal{Y})\|)\right).
\end{align}
As long as $\|P_T(\mathcal{Y})\|<1$, $\|U-P_T(\mathcal{Y}) \|_{\Fr}$ is small enough and $ \| P_T(Z)    \|_{*} $ is not too large in relation to $\|P_{T^\top}(Z)\|_*$, the solution will thus be unique.

In~\cite{IMCtheory1} these ideas are extended to the case where side information matrices $X,Y$ \textit{with orthonormal columns} is provided. The key here is that with this assumption on the columns, $\|XMY^\top\|_*=\|M\|_*$ for any matrix $M$, so that most of the arguments above still hold with minor modification, even after replacing the projection operator $P_T$ by its inductive analogue $P_{T}(Z)=P_XZP_F+P_EZP_B-P_EZP_F$.

\subsection{Removing the homogeneity assumption: proof strategy}

In our case, where $X,Y$ are arbitrary (they can without loss of generality be assumed to have orthogonal columns, though not necessarily of norm $1$), it is no longer true that  $\|XMY^\top\|_*=\|M\|_*$ for any $M$. To tackle this issue, we define a norm $\|Z\|_{\nuc}$ on the set of matrices $\R^{m\times n}$ which equals the minimum possible nuclear norm of a matrix $M$ such that $XMY^\top=Z$: 
\begin{align}
\|Z\|_{\nuc}&= \min\left( \|M\|_{*}\quad : XMY^\top =Z      \right).
\end{align}

A key observation is that both this norm and \textit{its dual} can be computed easily. Indeed, it is easy to see that  $\|Z\|_{\nuc}=\Sigma_1^2 X^\top Z Y \Sigma_2^2$ where $\Sigma_1,\Sigma_2$ are matrices containing the singular values of $X,Y$. Furthermore, we also show in the supplementary that in fact the dual norm $\|\nbull\|_{\spec}$ is simply the spectral norm of the matrix $X^\top R Y$. These modifications mean that during the proof, we must manipulate 5 different norms ($\|\nbull\|,\|\nbull\|_*,\|\nbull\|_{\spec},\|\nbull\|_{\nuc}$ and $\|\nbull\|_{\Fr}$), sometimes incurring factors of the smallest singular value $\sigma_0$ of $X,Y$. 

We note that removing the homogeneity assumption has consequences in the proofs, including the need for a stronger incoherence assumption.

\subsection{Fast decay in low-noise settings: proof strategy}

In addition, we need to account for the noise, thus instead of perturbing the matrix $R$ only by a matrix $Z$ with $P_{\Omega}(Z)=0$, we also perturb it by a matrix $H$ with $P_{\Omega^\top}(H)=0$ corresponding to the difference between the recovered matrix and the ground truth on the observed entries. Thus our recovered matrix, the solution to algorithm~\eqref{IMC}, $\widehat{R}$, can be written $\widehat{R}=R+H+Z$. 

Our next step is to perform a perturbed version of the calculation in equation~\eqref{eq:recht} taking into account the difference $H=P_{\Omega}(\widehat{R}-R)$. This is the calculation performed in the proof of Lemma~\ref{lem:thekey}. As previously we write $U$ for a subgradient  of $\|R\|_{\nuc}$ and $U^\top$ for a subgradient of $\|P_T(Z)\|_{\nuc}$. We start by expressing $\|\widehat{R}\|_{\nuc}$ as $\left\langle R+H+Z , U+U^\top      \right\rangle $ and after some calculations we obtain the following conclusion: 

\begin{align}
\label{eq:ourkey}
\|R\|_{\nuc}&\geq \|\widehat{R}\|_{\nuc}\nonumber \\
&\geq   \|R\|_{\nuc}-2\|H\|_{\nuc}+\frac{1}{4} \|P_{T^{\top}} (Z)\|_{\nuc}, 
\end{align}
which holds as long as several concentration phenomena occur (which will happen with high probability as long as $N$ is large enough).

Our next step is to bound $\|H\|_{\nuc}$. With high probability, the noisily observed entries of $R$ on $\Omega$ (the $R_\xi+\zeta_{\xi}$ ) are close to the actual entries $R$, which in turn implies that the entries of $H$ will not be too large (see the beginning of the proof of Theorem~\ref{thm:mainwithconditioningnumber}).

This yields a bound of order $\widetilde{O}(\sqrt{N}\nu)$ for $\|H\|_{\nuc}$, and then via equation~\eqref{eq:ourkey}, on $\|P_T(Z)\|_*$. Together with further modifications, this eventually yields a bound on the nuclear norm of $Z+H=\widehat{R}-R$. This means that our perturbed version of the exact recovery results places the recovered matrix $\widehat{R}$ inside of the  smaller function class of matrices within a bounded spectral norm of the ground truth matrix. At this point, we can leverage classical results on the Rademacher complexity of the function class of matrices with bounded nuclear norm (see Lemma~\ref{lem:coarserad}  below for the inductive version we use in practice) to further bound the generalization gap. Several further steps are needed to process the final result into an elegant formula that holds for any value of $N$. The details are in the supplementary material.

\begin{lemma}[\citealp{mostrelated}]
	\label{lem:coarserad}
	
	The function class $\left \{XMY^\top  : \|M\|_{*}\leq \mathcal{M}   \right \}$ satisfies
	\begin{align}
	\rad(\mathcal{F}_{\mathcal{M}})\leq \mathbf{x}\mathbf{y}\mathcal{M} \sqrt{\frac{1}{N}},
	\end{align}
	where $\mathbf{x}:=\|X^\top\|_{2,\infty}$ and $\mathbf{y}:=\|Y^\top\|_{2,\infty}$.
\end{lemma}

\begin{proof}
	Follows directly from Theorem 1 in~\cite{Kakade}, together with the duality between the nuclear and spectral norms~\cite{NuclearSpectreDual}. Cf. also~\cite{mostrelated}.
\end{proof}

\begin{figure*}
    \centering
    \includegraphics[width=0.68\linewidth]{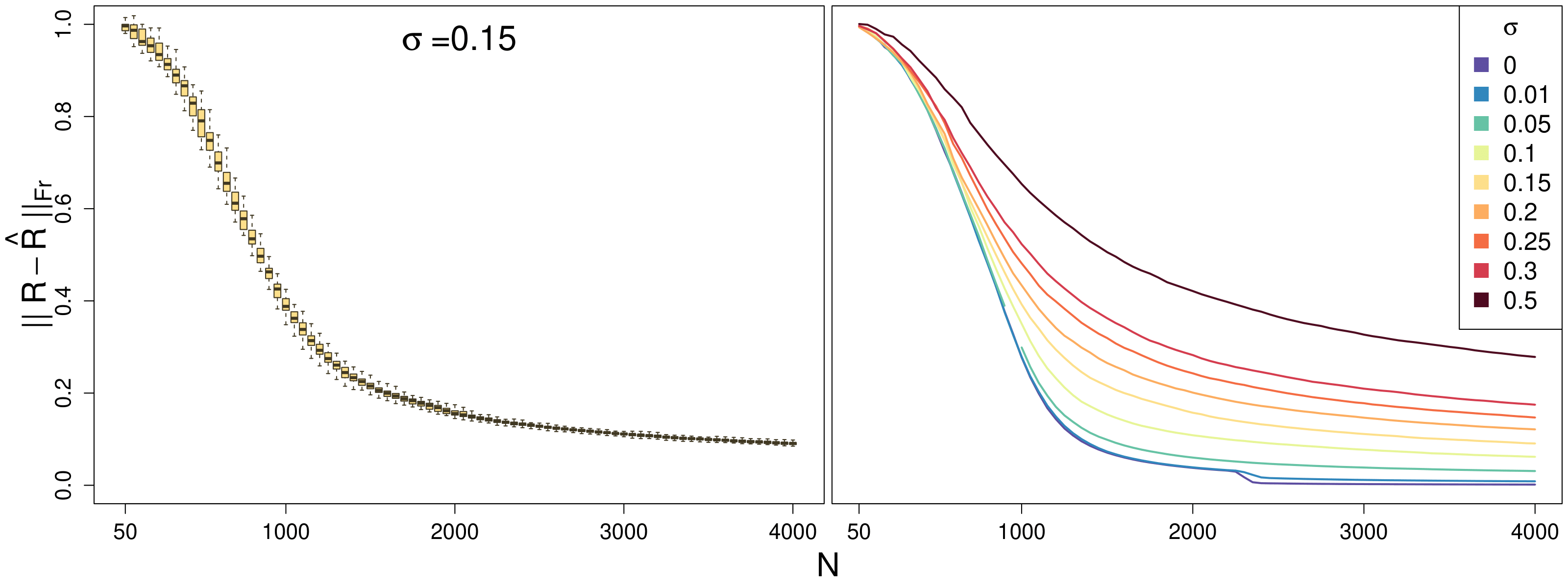}
    \caption{$\|\widehat{R}-R\|_{\Fr}$ as a function of $N,\sigma$}
    \label{omegaline}
\end{figure*}

\section{Experiments}

In this paper, we have posited that an accurate understanding of the sample complexity landscape of inductive matrix completion requires treating the noise component differently from the ground truth entries for the purposes of complexity. In this section we present the experiments we ran to confirm that a two-phase phenomenon as suggested by our bounds does in fact occur in practice. 

We considered random matrices of size $100\times 100$ and of rank $10$\footnote{To generate such a random matrix, we generate matrices $U,V\in\R^{100\times 10}$ with i.i.d. gaussian entries, then we form the matrix $UV^\top$ and we normalize it to have Frobenius norm $100$.}, and created random orthonormal side information of rank $40$, ensuring that the singular vectors of the ground truth matrix are in the span of the relevant side information, but with the orientation being otherwise uniformly random. The ground truth matrices were normalized to have Frobenius norm 100, and we then added i.i.d. $N(0,\sigma^2)$ gaussian noise to each observation. We performed classic inductive matrix completion (with the square loss) on the resulting training set, cross-validating the parameter $\lambda$ on a validation set, and evaluated the RMSE distance between the resulting trained matrix and the ground truth. We performed this whole procedure for a wide range of different values for the number of samples $N$. For each value of $N$ we perform the procedure on $40$ different random matrix and side information.

The results are presented in Figure~\ref{omegaline} below. The graph on the left contains box plots for our simulation with $\sigma=0.15$ whilst the graph on the right presents our results, averaged over the $40$ simulations, for several values of $\sigma$.



As can be observed in the figure, the graph of the error as a function of $N$ is not convex, despite the fact that traditional approximate recovery bounds $\widetilde{O}(1/\sqrt{N})$ are convex. Instead, the graph looks like a sigmoid: we can clearly observe a thresholding phenomenon where the performance is very poor initially, but very quickly improves past a minimum number of entries. 
Furthermore, as can be expected, after the threshold is crossed the error decreases slowly (at least as $\widetilde{O}(\frac{\sigma}{\sqrt{N}})$ as per the bounds in Theorem~\ref{thm:approxOnotation} above), confirming that inductive matrix completion in low-noise settings exhibits a two-phase phenomenon matching our theoretical results. Furthermore, the fact that the post threshold error curve scales as $\sigma$ is also apparent from the graphs.

\section{Conclusion and future directions}

In this paper, we have  studied  \textit{Inductive Matrix Completion} with nuclear norm regularisation in low-noise regimes. Our first contribution is an exact recovery result which generalizes the existing ones to the case where the side information is no longer assumed to be orthonormal, and to an arbitrary sampling regime (previously, the number of samples was required to be bounded \textit{above} by $\widetilde{O}(abr)$). Our second contribution consists in generalization bounds composed of two components: (1) the requirement that the number of samples should exceed a given threshold and (2) a term of order $\widetilde{O}(\sdv\sigma_0^{-2}a^{3/2}\sqrt{b}\log^3(N/\Delta)\sqrt{\frac{1}{N}})$ (ignoring incoherence constants and other constant quantities), which is directly proportional to the subgaussianity constant $\sigma$ of the noise. In particular, the result forms a bridge between exact recovery results and approximate recovery results: at the regimes where exact recovery is possible, the error converges to zero when the noise converges to zero. 

We believe our result and proof strategy open the door to a new and unexplored direction of research. Possible future directions include improving the dependence on $N$ from $1/\sqrt{N}$ to $\frac{1}{N}$, extending the results to non-trivially non uniform distributions or providing analogues of our results for other low-rank learning problems such as density estimation~\cite{song14,ROOOB,kargas19,anandkumar14,ROBXIV,amiridi21a,amiridi21b} or more complex recommender systems models that involve implicit feedback or graph/cluster information~\citep{Zhang2020Inductive,autobomic,wu2020inductive,EASE,vanvcura2022scalable,daniel,GRAPH_AC}.
Improving the dependence on $a,b$ to match the scaling of the ERT is also a very ambitious and interesting aim.



\section*{Acknowledgements}
Rodrigo Alves thanks Recombee for supporting his research. Marius Kloft acknowledges support by the Carl-Zeiss Foundation, the DFG awards KL 2698/2-1, KL 2698/5-1, KL 2698/6-1, and KL 2698/7-1, and the BMBF awards 01|S18051A, 03|B0770E, and 01|S21010C.

\bibliography{Bibliographomic}

\clearpage
\onecolumn

\setcounter{secnumdepth}{0}

\numberwithin{equation}{section}
\numberwithin{theorem}{section}
\numberwithin{figure}{section}
\numberwithin{table}{section}
\renewcommand{\thesection}{{\Alph{section}}}
\renewcommand{\thesubsection}{\Alph{section}.\arabic{subsection}}
\renewcommand{\thesubsubsection}{\Roman{section}.\arabic{subsection}.\arabic{subsubsection}}
\setcounter{secnumdepth}{-1}
\setcounter{secnumdepth}{3}

\appendix

\section{Some Concentration Inequalities and Classic Results}

\begin{proposition}
    Let $\zeta_1,\ldots,\zeta_N$ be i.i.d. $\sdv^2$-subgaussian random variables, i.e. $\log(\E(\exp(\lambda \zeta)))\leq \frac{\lambda^2\sdv^2}{2}$ for all $\lambda$. 
    For any $\delta>0$ we have with probability $\geq 1-\delta$:
    \begin{align}
    \sum_{i=1}^N \zeta_i^2\leq 32 N \sigma^2 \log\left(\frac{1}{\delta}\right).
    \end{align}
\end{proposition}

\begin{proof}

By Theorem 2.1 in~\cite{concentration} (page 25) we have 
\begin{align}
\label{eq:concenfirst}
\E(\zeta^4)&\leq 2\times  2!(2\sdv^2)^2\leq 16\sdv^4 \quad \quad \text{and} \quad \quad \forall q\\
\E([\zeta^2]^q))&=\E([\zeta^2]_{+}^q))\leq q! [4\sdv^2]^q.
\label{eq:concenfirst2}
\end{align}

We will now apply Theorem 2.10 (Bernstein's inequality, page 37) from~\cite{concentration} to the random variable $\sum_{i=1}^N\zeta_i^2$. By equation~\eqref{eq:concenfirst} we have 
\begin{align}
\sum_{i=1}^N \E([\zeta_i^2]_+^2)\leq 16N\sdv^4.
\end{align}
Furthermore, we also have by equation~\eqref{eq:concenfirst2} (for all $q\geq 3$):
\begin{align}
\sum_{i=1}^N\E([\zeta_i^2]_{+}^q)&\leq Nq! [4\sdv^2]^q=q![16N\sdv^4][4\sdv^2]^{q-2}\\
&\leq \frac{q!}{2}[16N\sdv^4][8\sdv^2]^{q-2}.
\end{align}

Thus we can apply Theorem 2.10 from~\cite{concentration} with "$\nu$" being $[16N\sdv^4]$ and "c" being $[8\sdv^2]$. We obtain that with probability $\geq 1-\delta$ we have 
\begin{align}
\left|\sum_{i=1}^N\zeta_i^2-N\sdv^2\right|&\leq \sqrt{2[16N\sdv^4]\log\left(\frac{1}{\delta}\right)} +[8\sdv^2] \log\left(\frac{1}{\delta}\right)\leq 16\sdv^2 \log\left(\frac{1}{\delta}\right) \sqrt{N},
\end{align}
as expected.

\end{proof}

\begin{proposition}
     Let $\zeta$ be a  $\sdv^2$-subgaussian random variable, i.e. $\log(\E(\exp(\lambda \zeta)))\leq \frac{\lambda^2\sdv^2}{2}$ for all $\lambda$. We have 
     \begin{align}
        \Var(\zeta)\leq \sigma^2.
     \end{align}
\end{proposition}
\begin{proof}
Cf. Exercise 2.16  on page 49 of~\cite{concentration}.
\end{proof}

\begin{proposition}[Theorem 2.1 on page 8 of~\cite{localrad}]
   \label{prop:bousquet}
   Let $\mathcal{F}$ be a class of functions that maps $\mathcal{X}$  to $[a,b]$. Assume that there is some $r>0$ such that for all for all $f\in\mathcal{F}$, $\Var(f)\leq r$. Then for all $\delta>0$ we have with probability $\geq 1-\delta$:
    \begin{align}
    \sup_{f\in\mathcal{F}}\left[ \E(f)-\sum_{i=1}^N f(x_i)\right]\leq \inf_{\alpha > 0}\left[ 2(1+\alpha)\E (\widehat{\rad}_n(\mathcal{F})) +\sqrt{\frac{2r \log\left(\frac{1}{\delta}\right)}{N}} +(b-a)\left(\frac{1}{3}+\frac{1}{\alpha}\right) \frac{ \log\left(\frac{1}{\delta}\right)}{N} \right],
    \end{align}
    where $\widehat{\rad}_n(\mathcal{F})$ denotes the empirical rademacher complexity of $\mathcal{F}$. Furthermore, the same result holds for  $\sup_{f\in\mathcal{F}} \left[\sum_{i=1}^N f(x_i)-\E(f)\right]$.
\end{proposition}

We recall the following matrix Bernstein inequality. This version is a combination of Lemmas 3 and 4 from~\cite{IMCtheory1}, but similar results are well known~\cite{bookhighprob}.

\begin{lemma}
	\label{lem:bernsteinwithamax}
	Let $X_1,\ldots X_L$ be independent, zero mean random matrices with dimensions $m\times n$. Suppose also that for all $k\leq L$,  $\rho_k^2:=\max\left(\E\left(\left\|X_kX_k^{\top}\right \|\right),\E\left( \left\|X_k^{\top} X_k\right \|\right)\right)$ and $\|X_k\|\leq M$ almost surely. Then for all $\delta>0$ we have with probability $\geq 1/\delta$: 
	\begin{align}
	\left \| \sum_{k=1}^L X_k\right\| \leq \max \left( \sqrt{\frac{8}{3} \log\left(\frac{m+n}{\delta}\right)\sum_{k=1}^L \rho_k^2   } ,   \>\> \frac{8}{3}M   \log\left(\frac{m+n}{\delta}\right)    \right).
	\end{align}
\end{lemma}

\begin{lemma}[Cf also Proposition 3.3 in ~\cite{SimplerMC}]
	
	\label{lem:delta5}
	Assume we sample entries from $[m]\times [n]$ independently and uniformly at random. For any $\delta_5>0$, with probability $\geq 1-\delta_5$ the number of repetitions of a single entry is bounded by 
	\begin{align}
	&\frac{N}{mn}+ \max\left( \frac{8}{3}\log\left(\frac{2mn}{\delta_5}\right),\sqrt{\frac{8}{3}\log\left(\frac{2mn}{\delta_5}\right)\frac{N}{mn}}     \right)\\
	&\leq \frac{N}{mn}+\frac{8}{3}\log\left(\frac{2mn}{\delta_5}\right) \sqrt{\frac{N}{mn}}  =:	\tau_5.
	\end{align}
	In particular, as long as $N\leq mn$, the number of repetitions is bounded by $\tilde{\tau}_5:=5\log\left(\frac{2mn}{\delta_5}\right)\geq \tau_5$.
	
\end{lemma}

\begin{proof}
	Follows from Lemma~\ref{lem:bernsteinwithamax} applied to each ($1\times 1$)-dimensional entry of the matrix, together with a union bound over all $m\times n$ entries.
\end{proof}

Note that classic approximate recovery bounds for inductive matrix completion typically rely on the following result. 

\begin{lemma}[\cite{mostrelated}]
	\label{lem:coarserad}
	
	The function class $\left \{XMY^\top  \Big| \|M\|_{*}\leq \mathcal{M}   \right \}$ satisfies
	\begin{align}
	\rad(\mathcal{F}_{\mathcal{M}})\leq \mathbf{x}\mathbf{y}\mathcal{M} \sqrt{\frac{1}{N}},
	\end{align}
	where $\mathbf{x}:=\|X^\top\|_{2,\infty}$ and $\mathbf{y}:=\|Y^\top\|_{2,\infty}$.
\end{lemma}

\begin{proof}
	Follows directly from Theorem 1 in~\cite{Kakade}, together with the duality between the nuclear and spectral norms~\cite{NuclearSpectreDual}. Cf also~\cite{mostrelated,omic}.
\end{proof}

\section{Assumptions and first consequences}


\noindent \textbf{Running generic assumptions}: recall the following assumptions from the main paper: 
\begin{assumption}[Realizability]
	There exists a matrix $M_*\in  \R^{a\times b}$ such that  $R=XM_{*}Y^\top$.
\end{assumption}

\begin{assumption}[The noise is $\nu^2$ subgaussian]
	Recall that we assume the noise is subgaussian: $\E(\zeta)=0$ and $\log(\E(\exp(\lambda \zeta)))\leq \frac{\lambda^2\sdv^2}{2}$ for all $\lambda$.
	
\end{assumption}

\noindent \textbf{Incoherence assumptions:}
we will write $\widebar{X}$ and $\widebar{Y}$ for the matrices obtained by normalizing the columns of $X,Y$ and we will write $\Sigma_1,\Sigma_2$ for the diagonal matrices containing the singular values of $X,Y$. Similarly we will also write $\widebar{\widebar{X}}=\widebar{X}\Sigma_1$ etc.

We organize our incoherence assumptions slightly differently from the main paper to obtain the most general results possible:
\begin{assumption}
	We make the following assumption on the coherence of the column spaces of  $X,Y$: 
	\begin{align}
	\|\widebar{X}_{i,\nbull}\|&\leq \sqrt{\frac{\mu a }{m}}\quad (\forall i) \quad \quad \text{and} \\
	\|\widebar{Y}_{j,\nbull}\|&\leq \sqrt{\frac{\mu b }{n}} \quad (\forall j).
	\end{align}
\end{assumption}

\begin{assumption}
	
	We assume we have the following bound on the coherence of the ground truth matrix $R$: let $ADB^\top $ be the SVD of  the core matrix $M^*$, we have 
	\begin{align}
	\|\widebar{X}\Sigma_1^{-1} AB^\top \Sigma_2^{-1} \widebar{Y}\|_{\infty}\leq  \sqrt\frac{\mu_1 r}{mn} \sigma_0^{-2}.
	\end{align}
	
\end{assumption}

\noindent \textbf{Remark: } if the columns of $X,Y$ are normalised (as is assumed in~\cite{IMCtheory1}), we directly obtain the same assumption as in~\cite{IMCtheory1}. In the general case, it is not immediately clear how to deduce our assumption from theirs: our assumption requires some form of "joint" incoherence between the side information matrices $X,Y$ and the ground truth core matrix $M^*$.

Nevertheless, such an assumption can be reasonably expected to hold for many matrices. Indeed, it can be deduced as long as $X,Y$ and $M$ satisfy the stricter notion of incoherence used in the main paper, i.e.
\begin{align}
\|\bar{X}\|_{\infty}\leq \sqrt{\frac{\mu}{m} },\nonumber & \quad \quad \quad 
\|\bar{Y}\|_{\infty}\leq \sqrt{\frac{\mu}{n} }\nonumber \\
\|A\|_{\infty}\leq \sqrt{\frac{\bar{\mu}}{a}},\nonumber & \quad \quad \quad 
\|B\|_{\infty}\leq \sqrt{\frac{\bar{\mu}}{b}}.\label{eq:incoherencestrictseparate}
\end{align}

In that case 
\begin{align}
\|\widebar{X}\Sigma_1^{-1} AB^\top \Sigma_2^{-1} \widebar{Y}\|_{\infty}&\leq    \max_{i,j}\| \widebar{X}\Sigma_1^{-1}A \|_{i,\nbull}\| \widebar{Y}\Sigma_1^{-1}B \|_{j,\nbull}\\
&\leq r\sqrt{a}  \sigma_0^{-1}\|\widebar{X}\|_{\infty}\|A\|_{\infty }\sqrt{b}  \sigma_0^{-1} \|\widebar{Y}\|_{\infty}\|B\|_{\infty }  \\
&\leq r \sigma_0^{-2} \sqrt{a}\sqrt{\mu/m}\sqrt{\bar{\mu}/a}\sqrt{b}\sqrt{\mu/n}\sqrt{\bar{\mu}/b}\\
&\leq r \sigma_0^{-2} \sqrt{\bar{\mu}^2 \mu^2 \frac{1}{mn}},
\end{align}
yielding 

\begin{align}
\label{eq:mu1andmu}
\mu_1\leq  \bar{\mu}^2 \mu^2r.
\end{align}

Using~\eqref{eq:mu1andmu} in the lemmas and theorems in this supplementary allows one to obtain the results in the main paper, which assume the strict notion of in coherence above.

\noindent \textbf{Further notation:}
in addition to the notation introduced in the main paper, we will, similarly to~\cite{IMCtheory1}, define the following operators:  $P_X, P_Y$. Here $P_{X}:\R^{m}\rightarrow \R^m$ and $P_{Y}:\R^{n}\rightarrow \R^n$ are the projection operators onto the column subspaces of $X$ and $Y$ respectively. We assume without loss of generality that the columns of $X,Y$ are orthogonal and ordered with decreasing norm with the norm of the first column being $1$ and the norm of the last column being more than $\sigma_0$ . 

Let $E=\widebar{X}A$ and $F=\widebar{Y}B$. 
We will also write  (analogously to~\cite{IMCtheory1}) $P_T$ for the operator defined as follows
$$P_T(Z)=  P_{E}ZP_Y + P_X ZP_F-P_XZP_Y,$$  
as well as denote the operator $ P_{T^\top }$ by
$$ P_{T^\top }(Z)= P_{X^\top} Z P_{Y^\top}. $$

We also write $P_{\Omega}$ for the operator from $\R^{m\times n}$ to itself defined by $[P_{\Omega}(Z)]_{i,j}=h_{i,j}Z_{i,j}$, where $h_{i,j}=\#(k\leq N: \> \xi^k=(i,j))$ denotes the number of times that entry $(i,j)$ was sampled. Note that $P_{\Omega}$ is the sum of $N$ i.i.d. samples from a uniform distribution over the operators $P_{\{(i,j)\}}$ for $(i,j)\in [m]\times [n]$, and it is not necessarily a projection operator (because the same entry can be sampled several times).


Recall the optimization problem considered is the following: 
\begin{align}
\label{eq:optreallife}
\min \frac{1}{N}\sum_{\xi\in\Omega} \left|[R_{\xi}+\zeta_\xi]-XMY^\top \right|^2+\lambda \|M\|_{*},
\end{align}
where $\lambda$ is a regularization parameter.

\noindent \textbf{The norms $\|\nbull\|_{\nuc}$ and $\|\nbull\|_{\spec}$:}
now, in order to better take into account the non homogeneous spectrum of $X$ and $Y$, we define two norms $\|\nbull\|_{\nuc}$ and $\|\nbull\|_{\spec}$ on the space $\R^{m\times n}$: 

\begin{align}
\|Z\|_{\nuc}&= \min\left( \|M\|_{*}\quad \Big  | XMY^\top =Z      \right)   = \left\|\widebar{\widebar{X}}^\top R \widebar{\widebar{Y}} \right  \|_{*} \\
\|Z\|_{\spec}&= \|X^\top RY\|.
\end{align}

One of the key aspects of our proof is that the above norms are dual to each other, and therefore the Taylor decomposition around the ground truth still has similar properties as in the case studied in~\cite{IMCtheory1}.

\begin{lemma}
	The norms $\|\nbull\|_{\nuc} $ and $\|\nbull\|_{\spec} $ are dual to each other (with respect to the Frobenius inner product on $\R^{m\times n}$). 
\end{lemma}
\begin{proof}
	
	Define $\|\nbull\|_{\dualnorm}$ on $\R^{m\times n}$ to be the dual norm to $\|\nbull\|_{\nuc}$. We will show that $\|\nbull\|_{\dualnorm }=\|\nbull\|_{\spec}$. 
	
	Let $B\in \R^{m\times n}$, we have 
	\begin{align}
	\|B\|_{\dualnorm}&:=\sup_{\|A\|_{\nuc}=1}\left\langle A,B\right\rangle\nonumber \\
	&=\sup_{M\in \R^{a\times b}, \|M\|_*=1} \left\langle [XMY^\top], B\right \rangle \nonumber\\
	&=\sup_{M\in \R^{a\times b}, \|M\|_*=1} \left\langle M, X^\top B Y  \right \rangle_{a\times b} \nonumber\\
	&=\|X^\top BY\|\\
	&=\|B\|_{\spec},
	\end{align}
	where the first line follows by the definition of $\|\nbull\|_{\dualnorm }$, the second line follows from the definition of $\|A\|_{\nuc}$, the third line follows from direct calculation and the properties of the trace and inner products, the fourth  line follows from the duality between the (ordinary) nuclear norm and the (ordinary) spectral norm on the space $\R^{a\times b}$ (cf. e.g. ~\cite{fazel}), and the last identity follows from the definition of the norm $\|B\|_{\spec}$. This concludes the proof. 
	
\end{proof}

\noindent \textbf{Remark:} It  is worth making a few observations which are necessary to understand the differences between our proofs and the exact recovery proofs in~\cite{IMCtheory1}, which only apply to the case where the columns of $X$ and $Y$ are normalised.

Let $\bar{X}$ and $\bar{Y}$  be the matrices obtained from $X$ and $Y$ by normalizing the columns.  For a matrix $R$, the canonical way of expressing it as $R=XMY^\top$, is by setting  $M=\Sigma_1^{-1}\bar{X}^\top R\bar{Y}\Sigma_2^{-1}$ where $\Sigma_X$ (resp. $\Sigma_Y$) denotes the diagonal matrix whose entries are the norms of the columns of $X$ (resp. $Y$). We then have $\|R\|_{\spec}=\|M\|_*$. In particular, writing $P_X$ (resp. $P_Y$) for the projection operator on the space spanned by the columns of $X$ (resp. $Y$), i.e. $P_X=\bar{X}^\top\bar{X}$ and  $P_Y=\bar{Y}^\top\bar{Y}$ , we have $\|M\|_*=\|R\|_{\nuc}\geq  \|P_XRP_Y\|_*=\|\bar{M}\|_*$ where $\bar{M}$ is defined by $\bar{X}^\top R\bar{Y}$, and the inequality can be strict (there is equality if the columns of $X$ and $Y$ are both normalized). 
On the other hand, by definition, we have $\|R\|_{\spec}=\|X^\top RY\|_{\sigma}= \|\Sigma_1\bar{X}^\top R\bar{Y}\Sigma_2\|_{\sigma}=\|\Sigma_2^2[\Sigma_1^{-1}\bar{X}^\top R\bar{Y}\Sigma_1^{-1}]\Sigma_2^2\|_{\sigma}=\|\Sigma_1^2M\Sigma_2^2\|_{\sigma}\leq \|M\|_\sigma$ (and also $\|R\|_{\spec}=\|\Sigma_1\bar{X}^\top R\bar{Y}\Sigma_2\|_{\sigma}=\|\Sigma_X\bar{M}\Sigma_2\|_{\sigma}\leq \|\bar{M}\|=\|P_XRP_Y\|$). In other words, directions in $R$ which correspond to columns of $X$ or $Y$ with small norms will make $\|R\|_{\nuc}$ large, but $\|R\|_{\spec}$ small, which is reasonable by duality.

It further makes sense intuitively: directions corresponding to small norms for $X$ and $Y$ correspond to a strong prior that the ground truth matrix $R$ doesn't have a significant component in these directions. Thus, if a recovered matrix $Z$ 
presents with a large component in these directions, then it must have a high $\|\nbull\|_{\spec}$  norm to discourage it. Then, since the optimisation algorithm already discourages such solutions, the values of the components of the dual certificate in those directions is also less important.


\section{Main technical lemmas}

As explained in the main document, we are especially interested in the following optimization problem: 
\begin{align}
\label{eq:optreallife}
\min \frac{1}{N}\sum_{\xi\in\Omega} \left|[R_{\xi}+\zeta_\xi]-XMY^\top \right|^2+\lambda \|M\|_{*}.
\end{align}

Note first that this optimization problem is trivially equivalent to the following:
\begin{align}
\label{eq:optreallifeeqiv}
\min_M \frac{1}{N}\left\|\left[R+\zeta-XMY^\top \right]  \circ \mathcal{H} \right\|_{\Fr}^2+\lambda \|M\|_{*},
\end{align}
	where $\circ$ denotes the Hadamard (entry-wise) product between matrices and $\mathcal{H}\in\N^{m\times n}$ is the matrix containing the number of times that each entry was observed. By abuse of notation we write $\zeta\in\R^{m\times n}$ for the matrix whose $(i,j)$ entry is the average of the noise of the observations corresponding to entry $(i,j)$ (for unobserved entries we set $\zeta_{i,j}$ to zero).

		Let  $ \widehat{R}$ be the solution to the constrained optimization problem~\eqref{eq:optreallife} . We will write $\widehat{R}=Z+H+R\in \R^{m\times n}$  with $R$ being the ground truth, $H\in P_{\Omega}(\R^{m\times n})$ and  $Z\in P_{\Omega^\top}(\R^{m\times n})$.  Our strategy is to show that if the noise matrix $\zeta$ is small enough and if the number of samples is large enough, the dual certificate of $R$ with respect to the norm $\|\nbull\|_{\nuc}$ can be well captured by a matrix in the span of the observed entries, which allows us to show that $Z$ must have small nuclear norm, from which it later follows that $H$ also has low nuclear norm. 
	
\begin{lemma}
	
	\label{lem:thekey}
	Let $\tau>0$. 

Assume that the regularization parameter $\lambda$ satisfies 
	\begin{align}
	\label{eq:lambdacond}
	\frac{\sdv \sigma_0^2}{C\sqrt{aN}}\leq \lambda \leq \frac{C\sdv \sigma_0^2}{\sqrt{aN}},
	\end{align}
	for some constant $C$ and that $Z$ satisfies 
	
	\begin{align}
	\label{eq:nottoomuchcrazystuff}
	\|P_T(Z)\|_{\Fr}\leq   \sqrt{\frac{3\tau a}{r}}    \|P_{T^\top}(Z)\|_{\Fr}
	\end{align}
for some given $\tau>0$. 	Assume also that 
	\begin{align}
	\label{eq:nooutlier}
	\|\zeta\|_{\infty}\leq B\sdv/\sqrt{\kappa}
	\end{align}
	for some $B>1$. 
	
	Now, let ${\du}$ be a dual certificate of $R=P_T(R)$ with respect to the norm $\|\nbull\|_{\nuc}$, which is to say that $\|{\du}\|_{\spec}=1$ and $\left\langle R,{\du}\right\rangle =\|R\|_{\nuc}$. Similarly, let ${\dut}$ be a dual certificate for $P_{T^\top}(Z)$. 
	Assume also that there exists a $\mathcal{Y}$ in the image of $P_{\Omega}$ such that \begin{align}
	\label{eq:dualclose}
	    \|P_T(\mathcal{Y})-\du\|_{\Fr}\leq \frac{1}{4}\sqrt{\frac{r}{3a\tau}}\frac{1}{\sigma_0^{-2}}
	\end{align} and \begin{align}
	\label{eq:complesmall}
	   \|P_{T^\top}(\mathcal{Y})\|_{\spec}\leq \frac{1}{2}.
	\end{align} (Here, as usual, $\sigma_0$ denotes the smallest singular value of $X$ or $Y$.)
	
	We have 
	\begin{align}
	\label{eq:resultZ}
	\|Z\|_*\leq 32C aB^2\sigma_0^{-2}\sdv \sqrt{3\tau N/\kappa}.
	\end{align}

	Furthermore, we also have 
	\begin{align}
	   \|H\|_{*}\leq  6C\sqrt{a}B\sdv \sqrt{N/\kappa}.
	\end{align}
	
	In particular, this implies that 
	\begin{align}
	   \left\|\widehat{R}-R\right\|_{*}=\|H\|_{*}+\|Z\|_{*}\leq  70 C aB^2\sigma_0^{-2}\sdv \sqrt{\tau N/\kappa}.
	\end{align}
\end{lemma}

\begin{proof}

	Since $\widehat{R}$ is the solution to the optimization problem~\eqref{eq:optreallifeeqiv}, we have
	\begin{align}
	\label{eq:new}
	\lambda \|\widehat{R}\|_{\nuc} +\frac{1}{N}\|P_{\Omega}(\widehat{R}-R-\zeta)\|_{\Fr}^2\leq \lambda \|R\|_{\nuc}+\frac{1}{N}\|\zeta\circ \mathcal{H}\|_{\Fr}^2.
	\end{align}

		Note that below by abuse of notation we write $\|P_{\Omega}(\zeta)\|_{\Fr}^2$ for $\sum_{o=1}^N\zeta_{o}^2=\|\zeta\circ \mathcal{H}\|_{\Fr}^2$ (even when some entries have been sampled several times).

Now note that we have
	\begin{align}
& \|\widehat{R}\|_{\nuc}\nonumber \\
	&\geq \left\langle R+H+Z , \du+\dut      \right\rangle \nonumber \\
	&\geq \|R\|_{\nuc}+\left\langle H, \du+\dut      \right\rangle + \left\langle Z, \du+\dut      \right\rangle\nonumber \\
	&\geq  \|R\|_{\nuc}+\left\langle H, \du+\dut      \right\rangle + \left\langle Z, -P_T(\mathcal{Y})-P_{T^\top}(\mathcal{Y})+ \du+\dut      \right\rangle\label{zy0} \\
	&\geq  \|R\|_{\nuc}+\left\langle H, \du+\dut      \right\rangle + \left\langle Z, \du -P_T(\mathcal{Y})\right\rangle + \left\langle Z, \dut  -P_{T^\top}(\mathcal{Y})    \right\rangle\nonumber \\
	&\geq  \|R\|_{\nuc}-2\|H\|_{\nuc}+ \left\langle Z, \du -P_T(\mathcal{Y})\right\rangle + \left\langle Z, \dut  -P_{T^\top}(\mathcal{Y})    \right\rangle\label{Uunit} \\
	&\geq  \|R\|_{\nuc}-2\|H\|_{\nuc}- \|P_T(Z)\|_{\Fr}\|\du -P_T(\mathcal{Y})\|_{\Fr} +\left\langle Z, \dut  -P_{T^\top}(\mathcal{Y})    \right\rangle \label {duality} \\
	&\geq   \|R\|_{\nuc}-2\|H\|_{\nuc}- \|P_T(Z)\|_{\Fr}\|\du -P_T(\mathcal{Y})\|_{\Fr} +\|P_{T^{\top}} (Z)\|_{\nuc} -\|P_{T^{\top}} (Z)\|_{\nuc} \|P_{T^{\top}}(\mathcal{Y}) \|_{\spec}   \label{duality+defineUtop}\\
	&\geq   \|R\|_{\nuc}-2\|H\|_{\nuc}- \|P_T(Z)\|_{\Fr}\|\du -P_T(\mathcal{Y})\|_{\Fr}+\|P_{T^{\top}} (Z)\|_{\nuc} \left[1- \|P_{T^{\top}}(\mathcal{Y}) \|_{\spec} \right]  \nonumber \\
	&>   \|R\|_{\nuc}-2\|H\|_{\nuc}- \|P_T(Z)\|_{\Fr}\frac{1}{4}\sqrt{\frac{r}{3a\tau}}\frac{1}{\sigma_0^{-2}}+\frac{1}{2} \|P_{T^{\top}} (Z)\|_{\nuc} \label{assumptions} \\
	&>   \|R\|_{\nuc}-2\|H\|_{\nuc}- \sqrt{\frac{3\tau a }{r}}\|P_{T^\top}(Z)\|_{\Fr}\frac{1}{4}\sqrt{\frac{r}{3a\tau}}\frac{1}{\sigma_0^{-2}}+\frac{1}{2} \|P_{T^{\top}} (Z)\|_{\nuc} \label{assumptionsagain} \\
	&>   \|R\|_{\nuc}-2\|H\|_{\nuc}- \frac{1}{4\sigma_0^{-2}}  \|P_{T^\top}(Z)\|_{\Fr}+\frac{1}{2} \|P_{T^{\top}} (Z)\|_{\nuc} \label{assumptionsagain1} \\
	&\geq   \|R\|_{\nuc}-2\|H\|_{\nuc}+\frac{1}{4} \|P_{T^{\top}} (Z)\|_{\nuc}, \label{assumeZ} 
	\end{align}

\noindent	where at equation~\eqref{zy0}, we have used the fact that $\langle Z,\mathcal{Y}\rangle=0$ (since $P_\Omega(Z)=0$ and $P_{\Omega}(\mathcal{Y})=\mathcal{Y}$), at equation~\eqref{Uunit} we have used the fact that $\|\du\|_{\spec}=1=\|\dut\|_{\spec}=1$, at equation~\eqref{duality} we have used the duality between the norms $\|\nbull\|_{\spec}$ and  $\|\nbull\|_{\nuc}$, at equation~\eqref{duality+defineUtop} we have used the duality and the definition of $\dut$, at equation~\eqref{assumptions} we have used the assumptions on $\mathcal{Y}$ from the Lemma statement, at equation~\eqref{assumptionsagain} we have used equation~\eqref{eq:nottoomuchcrazystuff}, at equation~\eqref{assumptionsagain1} we have simply simplified the expression, and at equation~\eqref{assumeZ} we have used the assumptions on $Z$ as well as the fact $\|P_{T^\top}(Z)\|_{\Fr}\leq \sigma^{-2}_0\|P_{T^\top}(Z)\|_{\nuc}.$

	From the above equation together with equation~\eqref{eq:new} we obtain: 
	
	\begin{align} 
	\label{eq:useforH}
	\lambda \|R\|_{\nuc}+\frac{1}{N}\|P_{\Omega}(\zeta)\|_{\Fr}^2
&	\geq\lambda \|\widehat{R}\|_{\nuc} +\frac{1}{N}\|P_{\Omega}(\widehat{R}-R-\zeta)\|_{\Fr}^2\\
&\geq  \lambda \left[\|R\|_{\nuc}-2\|H\|_{\nuc}+\frac{1}{4} \|P_{T^{\top}} (Z)\|_{\nuc}\right]+\frac{1}{N}\|P_{\Omega}(\widehat{R}-R-\zeta)\|_{\Fr}^2,
	\end{align}

	from which it follows that 
	\begin{align}
	\label{eq:thefork}
	   \|P_{T^{\top}} (Z)\|_{\nuc}\leq 8\|H\|_{\nuc}+\frac{4}{\lambda N}\left[\|P_{\Omega}(\zeta)\|_{\Fr}^2-\|P_{\Omega}(H-\zeta)\|_{\Fr}^2\right].
	\end{align}
	
	Note also that 

	\begin{align}
 \|H\|_{\nuc}&=	   \|(H-\zeta)+\zeta\|_{\nuc}\leq \|(H-\zeta)\|_{\nuc} +\|\zeta\|_{\nuc}  \\
	   &\leq \sigma_0^{-2}\sqrt{a}\|(H-\zeta)\|_{\Fr} +\|\zeta\|_{\nuc} \\
	   &\leq	
	   \sigma_0^{-2}\sqrt{a/\kappa}\|(H-\zeta)\circ \mathcal{H}\|_{\Fr} +\|\zeta\|_{\nuc}.
	\end{align}

	Thus we can further write
	\begin{align}
	    \|P_{T^{\top}} (Z)\|_{\nuc}&\leq 8\|H\|_{\nuc}+\frac{4}{\lambda N}\left[\|P_{\Omega}(\zeta)\|_{\Fr}^2-\|P_{\Omega}(H-\zeta)\|_{\Fr}^2\right] \\
	    &\leq 8\left[ \sigma_0^{-2}\sqrt{a/\kappa}\|(H-\zeta)\circ \mathcal{H}\|_{\Fr} +\|\zeta\|_{\nuc}  \right] +\frac{4}{\lambda N}\left[\|P_{\Omega}(\zeta)\|_{\Fr}^2-\|P_{\Omega}(H-\zeta)\|_{\Fr}^2\right]\\
	    &\leq 8\|\zeta\|_{\nuc}  +\upsilon\left[8\sigma_0^{-2}\sqrt{a/\kappa}-\frac{4}{\lambda N} \upsilon\right] +\frac{4}{\lambda N}\|\mathcal{H}\circ \zeta\|_{\Fr}^2,
	\end{align}
	where we have used the notation $\upsilon:=\|[H-\zeta]\circ \mathcal{H}\|_{\Fr}$.
	
	Now, note that the maximum of the function $\upsilon(b-a\upsilon)$ is $\frac{b^2}{4a}$ (attained at $\frac{b}{2a}$). Thus, optimising the above over $\upsilon$ we obtain: 
\begin{align}
   \|P_{T^{\top}} (Z)\|_{\nuc}&\leq 8\|\zeta\|_{\nuc}  +\upsilon\left[8\sigma_0^{-2}\sqrt{a/\kappa}-\frac{4}{\lambda N} \upsilon\right] +\frac{4}{\lambda N}\|\mathcal{H}\circ \zeta\|_{\Fr}^2\\
   &\leq 8\|\zeta\|_{\nuc} +4a\sigma_0^{-4} \lambda N/\kappa+\frac{4}{\lambda N}\|\mathcal{H}\circ \zeta\|_{\Fr}^2\\
   &\leq 8\|\zeta\|_{\nuc}+4a\sigma_0^{-4} \lambda N/\kappa+\frac{4\sdv^2B^2}{\lambda \kappa }\label{eq:useB}\\
   &\leq 8\sigma_0^{-2}B\sdv\sqrt{aN/\kappa} +4a\sigma_0^{-4} \lambda N/\kappa+\frac{4\sdv^2B^2}{\lambda  \kappa}\label{eq:useBagain}\\
   &\leq 8\sigma_0^{-2}B\sdv \sqrt{aN/\kappa}  + 4C\sigma_0^{-2}\sdv \sqrt{aN}/\kappa\left[1+B^2\right] \leq 16CB^2\sigma_0^{-2}\sdv \sqrt{aN/\kappa},
  \label{eq:optlam}
\end{align}
	where at lines~\eqref{eq:useB} and~\eqref{eq:useBagain} we have used the condition that $\|\zeta\|_{\infty}\leq B\sdv/\sqrt{\kappa}$, at line~\eqref{eq:useBagain} we have used the fact that $\|\zeta\|_{\nuc}\leq \sigma_0^{-2}\|\zeta\|_{*}\leq \sigma_0^{-2}\sqrt{a}\|\zeta\|_{\Fr}\leq \sigma_0^{-2}B\sqrt{a}\sdv \sqrt{N/\kappa}$ and at the last line~\eqref{eq:optlam} we have used the condition~\eqref{eq:lambdacond} on $\lambda$, the fact that $\kappa\geq 1$ and Lemma~\ref{lem:optlam}.

	From the above equation we obtain that 
	\begin{align}
	\label{first}
	\|P_{T^{\top}} (Z)\|_{*}\leq  	\|P_{T^{\top}} (Z)\|_{\nuc}  \leq 16CB^2\sigma_0^{-2}\sdv \sqrt{aN/\kappa}.
	\end{align}

	With one more use of the assumptions on $Z$, we obtain

	\begin{align}
	\|P_T(Z)\|_{\Fr}\leq 	\sqrt{\frac{3\tau a}{r}} \|P_{T^\top}(Z)\|_{\Fr}\leq	\sqrt{\frac{3\tau a}{r}} \|P_{T^\top}(Z)\|_{*}   \leq 	\sqrt{3a\tau/r } \|P_{T^\top}(Z)\|_{\nuc}  \leq 16Ca \sqrt{3\tau/r}B^2\sigma_0^{-2}\sdv \sqrt{N/\kappa}.
	\end{align}

	Then 
	\begin{align}
	\label{second}
	\|P_T(Z)\|_{*}\leq \sqrt{r}\|P_T(Z)\|_{\Fr} \leq 16C aB^2\sigma_0^{-2}\sdv \sqrt{3\tau N/\kappa}.
	\end{align}

	Putting equations~\eqref{first} and~\eqref{second} together, we obtain the result~\eqref{eq:resultZ}.

	Regarding the bound on $\|H\|_{\nuc}$, note that by equation~\eqref{eq:useforH}
	\begin{align}
	 \kappa\|H-\zeta\|_{\Fr}^2 \leq  \|[H-\zeta]\circ\mathcal{H}\|_{\Fr}^2 &\leq \|\zeta\circ \mathcal{H}\|_{\Fr}^2+\lambda N \left[ \|R\|_{\nuc}-\|\widehat{R}\|_{\nuc}\right]\leq \kappa\|\zeta\|_{\Fr}^2+\lambda N \left[ \|R\|_{\nuc}-\|\widehat{R}\|_{\nuc}\right].
	\end{align}
	Hence
		\begin{align} 
& \kappa\|H-\zeta\|_{\Fr}^2 \\&\leq \kappa\|\zeta\|_{\Fr}^2+\lambda N \left[ \|R\|_{\nuc}-\|\widehat{R}\|_{\nuc}\right]\\
 &\leq \kappa\|\zeta\|_{\Fr}^2+\lambda N\left[\|H\|_{\nuc}+\|Z\|_{\nuc}\right]\\
 &\leq \kappa \|\zeta\|_{\Fr}^2+\lambda N\left[\|\zeta\|_{\nuc}+\|H-\zeta\|_{\nuc}+\|Z\|_{\nuc}\right]\\
 &\leq \kappa NB^2\sdv^2/\kappa+\lambda N \sigma_0^{-2}\sqrt{a/\kappa}\sqrt{NB^2\sdv^2}+\lambda N  32Ca^{3/2}B^2\sigma_0^{-2}\sdv \sqrt{N\kappa}+\lambda N \sqrt{a} \sigma_0^{-2} \|H-\zeta\|_{\nuc} \label{eq:now}\\
 & \leq NB^2\sdv^2+	\frac{C\sqrt{N\kappa}\sdv \sigma_0^2}{\sqrt{a}}\sigma_0^{-2}\sqrt{a/\kappa}\sqrt{NB^2\sdv^2}   +\frac{C\sqrt{N\kappa}\sdv \sigma_0^2}{\sqrt{a}}32C a^{3/2}B^2\sigma_0^{-2}\sdv \sqrt{N\kappa} +\frac{C\sqrt{N\kappa}\sdv \sigma_0^2}{\sqrt{a}}\|H-\zeta\|_{\nuc}\label{eq:nownow}\\
 &\leq  NB^2\sdv^2+CNB\sdv^2+ 32C^2 N\sdv^2 a B^2+\frac{C\sqrt{N\kappa}\sdv \sigma_0^2}{\sqrt{a}}\|H-\zeta\|_{\nuc}\\
 &\leq 65C^2NB^2\sdv^2+\frac{C\sqrt{N\kappa}\sdv \sigma_0^2}{\sqrt{a}}\|H-\zeta\|_{\nuc}\label{eq:Bgeq1}\\
 &\leq  65C^2NB^2\sdv^2+C\sqrt{N\kappa}\sdv \|H-\zeta_{}\|_{\Fr},
	\end{align}
where at equation~\eqref{eq:now} we have used equation~\eqref{eq:nooutlier} as well as equation~\eqref{eq:resultZ}, at equation~\eqref{eq:nownow} we have used the conditions on $\lambda$ (i.e. inequalities~\eqref{eq:lambdacond}), at line~\eqref{eq:Bgeq1} we have used the fact that $B>1$, and at the last line we have used comparisons between different norms. 
	
From this it follows that 
\begin{align}
\|H-\zeta\|_{\Fr}&\leq C\sqrt{N\kappa}\sdv/2\kappa+\sqrt{C^2N\kappa\sdv^2+65C^2NB^2\sdv^2\kappa}/2\kappa\\
&\leq 5 C B\sdv \sqrt{N/\kappa},
\end{align}
from which it follows that 
\begin{align}
\|H\|_{*}\leq 5C \sqrt{a}B\sdv \sqrt{N/\kappa}+\|\zeta\|_{*}\leq 6C \sqrt{a}B\sdv \sqrt{N/\kappa},
\end{align}
as expected.
	
\end{proof}

\begin{lemma}
	\label{lem:asinxu}
	For any $\delta>0$, with probability $\geq 1-(5q+1)\delta$ as long as $N\geq q\bar{T}$ , the conditions~\eqref{eq:nottoomuchcrazystuff},~\eqref{eq:dualclose} and~\eqref{eq:complesmall} of Lemma~\ref{lem:thekey} hold (with $\tau= \frac{N}{mn}+\frac{8}{3}\log\left(\frac{2mn}{\delta}\right) \sqrt{\frac{N}{mn}}$).
	
	Here $\bar{T}=4\frac{\mu_1}{\mu} \sigma_0^{-4}T=\frac{128}{3}\mu\mu_1r(a+b)   \log\left(\frac{2mn}{\delta}\right)$, and 
	$q=\log\left[e^6 \sigma_0^8 a \log\left(\frac{mn}{\delta}\right) \right]$ .

\end{lemma}
\begin{proof}
	The proof will be provided below in Section~\ref{sec:proofofxu}.
\end{proof}


\begin{lemma}
\label{lem:optlam}
Let $a,b>0$ the minimum value of the function 
\begin{align}
\label{eq:mustwe}
\lambda a +\frac{b}{\lambda}
\end{align}
 over $\lambda>0$ is equal to $2\sqrt{ab}$, realised at $\lambda=\sqrt{\frac{b}{a}}$. 
Furthermore, as long as 
\begin{align}
\frac{1}{C}\sqrt{\frac{b}{a}}\leq \lambda\leq C\sqrt{\frac{b}{a}},
\end{align}
we have 
\begin{align}
\label{eq:mustwe}
\lambda a +\frac{b}{\lambda}\leq 2C\sqrt{ab}.
\end{align}
\end{lemma}
\begin{proof}
The result is standard and a trivial application of standard calculus. The derivative of the expression~\eqref{eq:mustwe}
is $a-\frac{b}{\lambda^2}$ which only cancels at $\lambda=\sqrt{b/a}$ as expected. As for the second statement, each term in the expression~\eqref{eq:mustwe} is bounded by twice its value for $\lambda=\sqrt{b/a}$.

\end{proof}

\begin{lemma}
	\label{lem:trivunion}
	For any $\delta>0$, with probability greater than $1-\delta$ all $(i,j)\in\Omega$, $|\zeta_{(i,j)}|\leq\frac{\sdv}{\sqrt{\kappa}} \sqrt{2\log(2N/\delta)}$. 
\end{lemma}

\begin{proof}
	This follows immediately from a simple union bound applied to our subgaussianity assumption on the noise. 
\end{proof}

\begin{lemma}
	\label{lem:sooomanysamples}
	With probability $\geq 1- K mn \exp(-\frac{N}{2Kmn}) $, each entry is sampled at least $K$ times. 
	
\end{lemma}

\begin{proof}
	
	Let $N_1=\left\lfloor\frac{N}{K}\right\rfloor$, and let us divide the first $KN_1$ samples into $K$ different groups $\{g_1,\ldots,g_K\}$.

	The probability that any fixed entry $(i,j)$ is not sampled in group $g_k$ (for any $k$) is $$\left(1-\frac{1}{mn}\right)^{N_1}\leq  \exp\left(-\frac{N_1}{mn}\right).$$
	Thus the probability that there is at least one entry which is not sampled in group $k$ is  less than $ mn \exp\left(-\frac{N_1}{mn}\right)$.
	
	By a union bound over all the groups, we get that with probability $\geq 1- K mn \exp\left(-\frac{N_1}{mn}\right) $, each entry is sampled at least once in each group (and in particular is sampled at least $K$ times over all). The result follows since $N_1=\left\lfloor\frac{N}{K}\right\rfloor\geq \frac{1}{2} \frac{N}{K}$.

\end{proof}

\section{Main results}

\begin{theorem}
	\label{thm:exact}
	Assume that the entries are observed without noise. 
	For any $\delta>0$ as long as 
	$$N\geq \log\left[e^6 \sigma_0^8 a \log\left(\frac{mn}{\delta}\right) \right]\sigma_0^{-4}\frac{128}{3}\mu\mu_1r(a+b)   \log\left(\frac{2mn}{\delta}\right),$$ with probability $$\geq 1-\delta\left[1+5\log\left[e^6 \sigma_0^{-8} a \log\left(\frac{mn}{\delta}\right) \right]\right],$$  we have that $XM_{\min}Y^\top=R$ where $M_{\min}$ is a solution to the following optimization problem:
	\begin{align}
	\label{eq:strict}
	M_{\min}\in\argmin \left(\|M\|_* \quad \text{s.t.} \quad \forall (i,j)\in\Omega,  [XMY^\top]_{i,j}=R_{i,j}\right).
	\end{align}

\end{theorem}
\begin{proof}

\textbf{Informal:} The condition~\eqref{eq:lambdacond} is trivially respected since $\sdv=0$ and $\lambda=0$.
Taking the limit as $\lambda\rightarrow 0$, the theorem follows immediately from
	Lemma~\ref{lem:asinxu} together with Lemma~\ref{lem:thekey} upon noting that in this case, the lack of noise implies that condition~\eqref{eq:nooutlier} holds with $B=0$, which shows $Z=0$ as expected. 
	
\noindent	\textbf{Formal:}
	By Lemma~\ref{lem:asinxu}, the conditions~\eqref{eq:nottoomuchcrazystuff},~\eqref{eq:dualclose} and~\eqref{eq:complesmall} of Lemma~\ref{lem:thekey} hold (with $\tau= \frac{N}{mn}+\frac{8}{3}\log\left(\frac{2mn}{\delta}\right) \sqrt{\frac{N}{mn}}$). Thus we can write, as in the proof of lemma~\ref{lem:thekey} (equation~\eqref{assumeZ}):
	\begin{align}
	\label{eq:simplified}
	 \|\widehat{R}\|_{\nuc}	&\geq   \|R\|_{\nuc}-2\|H\|_{\nuc}+\frac{1}{4} \|P_{T^{\top}} (Z)\|_{\nuc}\\
	 &\geq \|R\|_{\nuc}+\frac{1}{4} \|P_{T^{\top}} (Z)\|_{\nuc}.
	\end{align}
	Indeed, the relevant calculation in lemma~\ref{lem:thekey} doesn't rely on the optimization problem or the value of $\lambda$, and at the second line we have simply used the fact that by definition of the optimization problem~\eqref{eq:strict}, $H=0$. 
	
	Now, by definition of the optimization problem~\eqref{eq:strict} we have $\|\widehat{R}\|_{\nuc}\leq \|R\|_{\nuc} $, which together with equation~\eqref{eq:simplified} implies that $\|P_{T^{\top}} (Z)\|_{\nuc}=0$, which implies $P_{T^{\top}} (Z)=0$ and of course $\|P_{T^{\top}} (Z)\|_{\Fr}=0$. Then by equation~\eqref{eq:nottoomuchcrazystuff} (satisfied as stated above by  Lemma~\ref{lem:asinxu}), we also have $\|P_{T} (Z)\|_{\Fr}=0$, from which we finally obtain that $Z=0$ as expected. 
\end{proof}

\noindent \textbf{Remarks:}

\begin{enumerate}
	\item There is a dependence on the conditioning number of $X,Y$ both logarithmic in the high probability and non logarithmic in the threshold for $N$. The quadratic dependence matches that in~\cite{PIMC} (though they relied on a different optimization problem so the questions are different). 
	\item It is worth unpacking the log terms as well. We have a logarithmic term in $\log(1/\delta)$  (i.e. $\log \log (1/\delta)$) both in the failure probability and in the threshold for $N$. This comes from the definition of $\tau$ from Lemma~\ref{lem:delta5} (which controls the number of times the same entry can be sampled), which is necessary in lemma~\ref{lem:boundpt}.
	Note that a similar term is also present (implicitly) in the result of~\cite{SimplerMC} (cf. the logarithm in the term "$6\log(n_2)(n_1+n_2)^{2-2\beta}$" in the main theorem on page 2 (equation 1.3)). Similarly to our result, the log term comes from the use of proposition 3.3 on page 5 (which plays a similar role to our Lemma~\ref{lem:delta5}), and is used later in the proof of the main Theorem (1.1) on page 8. 
	\item Contrary to~\cite{SimplerMC}, we do not assume that $N\leq mn$. This results in extra logarithmic factors in the definition of $\tau$ from the use of Lemma~\ref{lem:sooomanysamples}, which only show up as constants in the end after taking the logarithm of $\tau$. This, together with our rather loose bounding of $\tau$ (aimed at a compact formula rather than the tightest bound) explains the higher implicit constant from the factor $\log(e^6\ldots)$.
\end{enumerate}


For the next theorem, we consider an arbitrary loss function $\loss$ which is assumed to be bounded by $\losb$, and $\lip-$Lipschitz continuous.

\begin{theorem}
	\label{thm:mainwithconditioningnumber}

	Assume that condition~\eqref{eq:lambdacond} on $\lambda$ holds. 
	For any $\delta_0,\delta_1,\delta_2,\delta>0$, with probability $$\geq 1-\delta_1-\delta_0-\delta_2-\delta\left[1+5\log\left[e^6 \sigma_0^{-8} a \log\left(\frac{mn}{\delta}\right) \right]\right],$$ as long as $$N\geq \log\left[e^6 \sigma_0^{-8} a \log\left(\frac{mn}{\delta}\right) \right]\frac{128}{3}\mu\mu_1r(a+b)   \log\left(\frac{2mn}{\delta}\right),$$  we have
	
		\begin{align}
	\label{eq:thefinal}
&\mathbb{E}_{(i,j)\sim \mathcal{U} } (\loss(\widehat{R}_{(i,j)},[R+\zeta]_{(i,j)}))  \leq 
	&500C\lip a^{3/2}\sqrt{b} \mu \sigma_0^{-2} \sdv \sqrt{\theta}\log\left(\frac{2N}{\delta_0}\right) \sqrt{\frac{\log(\frac{N}{2\delta_2})}{N}} +\losb\frac{4\log(1/\delta_1)}{3N},
	\end{align}
	where, as in the main paper, $\mathcal{U}$ is the uniform distribution on the entries $[m]\times [n]$ and $\theta$ denotes the logarithmic quantity 
	\begin{align}
	   \theta=2\log\left(\frac{N}{2\delta_2}\right)+\frac{8}{3}\log\left(\frac{2mn}{\delta_2}\right) \sqrt{2\log\left(\frac{N}{2\delta_2}\right)}.
	\end{align}

\end{theorem}

\begin{proof}

If $N\leq 2$, the result holds trivially. If $N\geq 3$, 
let 
\begin{align}
K_{\delta_2}=K=\left\lfloor\frac{N}{2mn\log(\frac{N}{2\delta_2})}\right\rfloor.
\end{align}
(In particular, if $N\leq mn$, we have $K=0$). Let also $\kappa:=\max(K,1)$. 
Note that we have 
\begin{align}
Kmn\exp\left(-\frac{N}{2Kmn}\right)&\leq \frac{N}{2mn\log(\frac{N}{2\delta_2})}mn \exp\left(-\frac{N}{2Kmn}\right)\\
&\leq \frac{N}{2\log(\frac{N}{2\delta_2})} \exp\left(-\frac{N}{\frac{N}{\log(\frac{N}{2\delta_2})}}\right)\\
&\leq \frac{\delta_2}{\log(\frac{N}{2\delta_2})}\leq \delta_2.
\end{align}
Hence, by Lemma~\ref{lem:sooomanysamples}, we have that each entry is sampled at least $K$ times with probability $\geq 1-\delta_2$, and below, we restrict ourselves to the high probability event where this occurs. 

	Then, by Lemma~\ref{lem:trivunion} we have with probability $\geq 1-\delta_0$ that 
	\begin{align}
	\label{eq:controlnoiseL1}
	\max_{(i,j)\in\Omega}|\zeta_{(i,j)}|\leq \frac{\sdv}{\sqrt{\kappa}} \sqrt{2\log(2N/\delta_0)},
	\end{align}
which means that condition~\eqref{eq:nooutlier} holds with $B=\sqrt{2\log(2N/\delta_0)}$ (note that since $\delta_0\leq 1$, $B>1$ as required). The condition~\eqref{eq:lambdacond} also holds by assumption, and we can use Theorem~\ref{thm:exact} with the value of $\kappa$ defined as above ($\max(K,1)$).

Furthermore by Lemma~\eqref{lem:asinxu} we have with probability $\geq 1-\delta\left[1+5\log\left[e^6 \sigma_0^{-8} a \log\left(\frac{mn}{\delta}\right) \right]\right]$ that the conditions~\eqref{eq:nottoomuchcrazystuff},~\eqref{eq:dualclose} and~\eqref{eq:complesmall} of Lemma~\ref{lem:thekey} hold (with $\tau= \frac{N}{mn}+\frac{8}{3}\log\left(\frac{2mn}{\delta}\right) \sqrt{\frac{N}{mn}}$). 
Hence by Lemma~\eqref{lem:thekey}, on the same high-probability event (w.p. $\geq 1-\delta_0-\delta_2-\delta\left[1+5\log\left[e^6 \sigma_0^{-8} a \log\left(\frac{mn}{\delta}\right) \right]\right]$), we have\begin{align}
\label{eq:fromabove}
 \left\|\widehat{R}-R\right\|_*\leq   70 C aB^2\sigma_0^{-2}\sdv\sqrt{\tau N/\kappa}.
\end{align}

	Now, by Lemma~\ref{lem:coarserad}, the Rademacher complexity of the function class $$\mathcal{F}_1:=\left\{ R+\widebar{X}M\widebar{Y}^\top \big| \|M\|_{*}\leq    70 C aB^2\sigma_0^{-2}\sdv \sqrt{\tau N/\kappa}\right\}=\left\{ R+Z\big| \|Z\|_{*}\leq   \ 70 C aB^2\sigma_0^{-2}\sdv \sqrt{\tau N/\kappa}\right\}$$ is bounded as 
	\begin{align}
	\label{eq:firstiteration}
	\rad( \mathcal{F}_1)\leq \frac{1}{\sqrt{N}}\mu\sqrt{\frac{ab}{mn}}  70 C aB^2\sigma_0^{-2}\sdv  \sqrt{\tau N/\kappa}=\mu\sqrt{\frac{ab}{\kappa mn}}  70 C aB^2\sigma_0^{-2}\sdv  \sqrt{\tau }.
	\end{align}

	Thus, by proposition~\ref{prop:bousquet},  with probability $\geq 1-\delta_1$ we have that for any $Z\in \mathcal{F}_1$ (and for any $\alpha>0$)
	\begin{align}
	\label{eq:firstrademach}
	&\mathbb{E}_{(i,j)\sim \mathcal{U}} (\loss(Z_{(i,j)},[R+\zeta]_{(i,j)})) -\frac{1}{N}\sum_{(i,j)\in\Omega} \loss(Z_{(i,j)},[R+\zeta]_{(i,j)})  \\
	&\leq 2(1+\alpha)\lip\mu\sqrt{\frac{ab}{\kappa mn}}  70 C aB^2\sigma_0^{-2}\sdv  \sqrt{\tau }+ \lip \sqrt{\frac{1}{mn}} 70 C a B^2\sigma_0^{-2}\sdv  \sqrt{\tau N/\kappa }\sqrt{2\frac{\log(1/\delta_1)}{N}}\\
	&+\losb[1/3+1/\alpha]\frac{\log(1/\delta_1)}{N},
	\end{align}
where we have used the bound on the Lipschitz constant, Lemma~\ref{lem:coarserad} and fact that the "variance" $r$ is bounded as
\begin{align}
r:=\frac{1}{mn}\sum_{(i,j)\in [m]\times [n]} \loss^2(Z_{(i,j)},[R+\zeta]_{(i,j)})))\leq \lip^2 \frac{1}{mn}\|Z-R\|^2_{\Fr}\leq \frac{1}{mn}\|Z-R\|^2_{*},
\end{align}
together with another use of the result from Lemma~\ref{lem:thekey}.

	Now, by equation~\eqref{eq:fromabove}, $\widehat{R}\in\mathcal{F}_1$. Hence, (w.p. $\geq 1-\delta_1-\delta_2-\delta_0-5q\delta$) we can apply equation~\eqref{eq:firstrademach} to $\widehat{R}$. Furthermore, since obviously $R\in\mathcal{F}_1$ we certainly have $$ \frac{1}{N}\sum_{(i,j)\in\Omega}  \loss(\widehat{R}_{(i,j)},[R+\zeta]_{(i,j)})\leq \frac{1}{N} \sum_{(i,j)\in\Omega} \loss(R_{i,j},[R+\zeta]_{(i,j)})\leq \lip B\sdv .$$

	Hence we can write (w.p. $\geq 1-\delta_1-\delta_2-\delta_0-5q\delta$), taking $\alpha=1$,

	\begin{align}
\mathbb{E}_{(i,j)\sim \mathcal{U}} (\loss(\widehat{R}_{(i,j)},[R+\zeta]_{(i,j)})) & \leq  4\lip\mu\sqrt{\frac{ab}{\kappa mn}}  70 C aB^2\sigma_0^{-2}\sdv  \sqrt{\tau }+ \lip \sqrt{\frac{1}{\kappa mn}} 70 C aB^2\sigma_0^{-2}\sdv  \sqrt{\tau  }\sqrt{2\log(1/\delta_1)}\\
	&\quad \quad \quad \quad \quad \quad \quad \quad \quad \quad \quad \quad \quad \quad \quad \quad \quad \quad \quad \quad \quad \quad \quad \quad \quad \quad  +\losb\frac{4}{3}\frac{\log(1/\delta_1)}{N}\\
	&=500C\lip a^{3/2}\sqrt{b} \mu \sigma_0^{-2} \sdv \sqrt{\tau}\log(2N/\delta_0) \sqrt{\frac{1}{\kappa mn}} +\losb\frac{4\log(1/\delta_1)}{3N}\label{eq:withmn}
	\end{align}
	where at the last lines we have simply plugged in the definition of $B=\sqrt{2\log(2N/\delta_0)}$.

Now, note that by the definition of $\kappa=\max(K,1)$, we have 
\begin{align}
\kappa  \geq 2\frac{N}{2mn\log(\frac{N}{2\delta_2})},
\end{align}
and hence 
\begin{align}
\kappa mn \geq \frac{N}{\log(\frac{N}{2\delta_2})}.
\end{align}
Plugging this back into equation~\eqref{eq:makingprogressy}, we have 
\begin{align}
&\mathbb{E}_{(i,j)\sim \mathcal{U}} (\loss(\widehat{R}_{(i,j)},[R+\zeta]_{(i,j)})) \leq 500C\lip a^{3/2}\sqrt{b} \mu \sigma_0^{-2} \sdv \sqrt{\tau}\log\left(\frac{2N}{\delta_0}\right) \sqrt{\frac{\log(\frac{N}{2\delta_2})}{N}} +\losb\frac{4\log(1/\delta_1)}{3N},\label{eq:withmn}
	\end{align}

where   \begin{align}
	\tau&= \frac{N}{mn}+\frac{8}{3}\log\left(\frac{2mn}{\delta_2}\right) \sqrt{\frac{N}{mn}}.
	\end{align}

	Now, assume first that $N\leq  2mn\log(\frac{N}{2\delta_2})$. In this case, we can write 
	
	\begin{align}
	\tau&= \frac{N}{mn}+\frac{8}{3}\log\left(\frac{2mn}{\delta_2}\right) \sqrt{\frac{N}{mn}}\\
	&\leq 2\log(\frac{N}{2\delta_2})+\frac{8}{3}\log\left(\frac{2mn}{\delta_2}\right) \sqrt{2\log(\frac{N}{2\delta_2})}\\
	&:=\theta.
	\end{align}
	
	On the other hand, if $N\geq   2mn\log(\frac{N}{2\delta_2})$, the $K\geq 1$ and (on our event with probability $\geq 1-\delta_2$), we have that each entry is sampled at least once. In this case, using the notation of Lemma~\ref{lem:thekey}, we have that 
	\begin{align}
	   \|\widehat{R}-R\|_{*}&=\|H\|_*\leq 6C \sqrt{a} B \sdv \sqrt{\frac{N}{\kappa}}\\
	   &\leq 70 C aB^2\sigma_0^{-2}\sdv\sqrt{ N/\kappa}.
	\end{align}
	Note this is the same expression as~\eqref{eq:fromabove} without the $\tau$. Hence, by the same calculation as above, we have (w.p. $\geq 1-\delta_1-\delta_2-\delta_0-5q\delta$): 
	
		\begin{align}
	\label{eq:makingprogressy1}
&\mathbb{E}_{(i,j)\sim \mathcal{U}} (\loss(\widehat{R}_{(i,j)},[R+\zeta]_{(i,j)}))  \leq 
	500C\lip a^{3/2}\sqrt{b} \mu \sigma_0^{-2} \sdv \log\left(\frac{2N}{\delta_0}\right) \sqrt{\frac{\log(\frac{N}{2\delta_2})}{N}} +\losb\frac{4\log(1/\delta_1)}{3N}.
	\end{align}
	This is the same expression as equation~\eqref{eq:withmn} without the $\tau$.

	From this and the fact that $\theta\geq 1$ it follows that equation~\eqref{eq:thefinal} also holds in this case. This concludes the proof.

\end{proof}

\section{Result with the absolute loss}
As an almost immediate consequence of Theorem~\ref{thm:mainwithconditioningnumber} we have the following corollary which gives a rate of $\frac{}{}$

\begin{corollary}

	\label{cor:mainwithabsoluteloss}

	Assume that condition~\eqref{eq:lambdacond} on $\lambda$ holds. 
	For any $\delta_0,\delta_1,\delta_2,\delta>0$, with probability $$\geq 1-\delta_1-\delta_0-\delta_2-\delta\left[1+5\log\left[e^6 \sigma_0^{-8} a \log\left(\frac{mn}{\delta}\right) \right]\right]$$ as long as $$N\geq \log\left[e^6 \sigma_0^{-8} a \log\left(\frac{mn}{\delta}\right) \right]\frac{128}{3}\mu\mu_1r(a+b)   \log\left(\frac{2mn}{\delta}\right)$$  we have
	
		\begin{align}
	\label{eq:thefinal}
&\mathbb{E}_{(i,j)\sim \mathcal{U}} \left|\widehat{R}_{(i,j)}-[R+\zeta]_{(i,j)}\right| \leq 
	700 C a^{3/2}\sqrt{b} \mu \sigma_0^{-2} \sdv \Theta \sqrt{\frac{1}{N}},
	\end{align}
	where 
	
	\begin{align}
	    \Theta&:=
	    2 \log\left(\frac{Nmn}{\delta_2^2}\right) \log\left(\frac{2N}{\delta_0}\right) \log\left(\frac{1}{\delta_1}\right).
	\end{align}

\end{corollary}

\begin{proof}

First note that by the same arguments as in the proof of Theorem~\ref{thm:mainwithconditioningnumber} we have 
\begin{align}
\|\widehat{R}-R\|_*\leq 70 C aB^2\sigma_0^{-2}\sdv\sqrt{\theta N/\kappa}.
\end{align}

In particular, 
\begin{align}
\label{eq:canreduce}
\|\widehat{R}-R\|_\infty \leq \|\widehat{R}-R\|_*\leq 70 C aB^2\sigma_0^{-2}\sdv\sqrt{\theta N/\kappa}.
\end{align}

Based on this we define the loss function $\ell$ by \begin{align}
    \ell(x,y):=\min(|x-y|,70 C aB^2\sigma_0^{-2}\sdv\sqrt{\theta N/\kappa}).
\end{align}

Applying Theorem~\ref{thm:mainwithconditioningnumber} together with equation~\eqref{eq:canreduce} we obtain: 

\begin{align}
&\mathbb{E}_{(i,j)\sim \mathcal{U}} \left|\widehat{R}_{(i,j)}-[R+\zeta]_{(i,j)}\right|= \mathbb{E}\ell\left(\widehat{R}_{(i,j)},[R+\zeta]_{(i,j)}\right)\\
&\leq 500C\lip a^{3/2}\sqrt{b} \mu \sigma_0^{-2} \sdv \sqrt{\theta}\log\left(\frac{2N}{\delta_0}\right) \sqrt{\frac{\log(\frac{N}{2\delta_2})}{N}} +\losb\frac{4\log(1/\delta_1)}{3N}\\
&\leq 500C a^{3/2}\sqrt{b} \mu \sigma_0^{-2} \sdv \sqrt{\theta}\log\left(\frac{2N}{\delta_0}\right) \sqrt{\frac{\log(\frac{N}{2\delta_2})}{N}} +[70 C aB^2\sigma_0^{-2}\sdv\sqrt{\theta N/\kappa}]\frac{4\log(1/\delta_1)}{3N}\\
&\leq 700 C a^{3/2}\sqrt{b} \mu \sigma_0^{-2} \sdv \sqrt{\theta} \log\left(\frac{2N}{\delta_0}\right) \log(1/\delta_1)\sqrt{\log(\frac{N}{2\delta_2})} \sqrt{\frac{1}{N}}\\
&\leq 700 C a^{3/2}\sqrt{b} \mu \sigma_0^{-2} \sdv \Theta \sqrt{\frac{1}{N}},
\end{align}
where we have replaced the value of $B=\sqrt{2\log(2N/\delta_0)}$. This concludes the proof.

\end{proof}

\section{Proofs of results in $O$ notation}

\begin{proof}[Proof of Theorem~\ref{thm:exactOnotation}]
	This follows immediately from Theorem~\ref{thm:exact} together with the computational Lemma~\ref{lem:changedelta} below. Indeed, recall that $\mu_1\leq \mu^4 r$, then to tackle the different $\delta,\Delta$s we can apply  Lemma~\ref{lem:changedelta} to see that we can set 
	\begin{align}
	\delta = \min\left(\frac{\frac{\Delta}{6K_1}}{\log\left(\frac{3K_1K_2}{\Delta}\right)},   \frac{\Delta}{3[1+5\log(e^6\sigma_0^{-8} a )]}       \right)
	\end{align}
	with 
	\begin{align}
	K_2=mn\\
	K_1=5,
	\end{align}
	as this will ensure that $\delta\left[1+5\log\left[e^6 \sigma_0^8 a \log\left(\frac{mn}{\delta}\right) \right]\right]\leq \frac{2}{3}\Delta$.
	
\end{proof}

\begin{proof}[Proof of Theorem~\ref{thm:approxOnotation}]
	Follows by the same arguments as the proof of Theorem~\ref{thm:exactOnotation} setting also $\delta_0=\delta_1=\delta_2=\Delta/9$.
\end{proof}

\begin{proof}[Proof of Theorem 3]
The proof follows directly from Lemma~\ref{lem:changedelta} and corollary~\ref{cor:mainwithabsoluteloss} in exactly the same way as Theorem~\ref{thm:approxOnotation} above.

\end{proof}
\subsection{Change of variables for $\delta$}

\begin{lemma}
	\label{lem:changedelta}
	Let $K_1,K_2>1$ and let $\Delta>0$. As long as 
	\begin{align}
	\delta\leq \frac{\frac{\Delta}{2K_1}}{\log\left(\frac{K_1K_2}{\Delta}\right)},
	\end{align}
	then we have 
	\begin{align}
	K_1\delta\log\left(\frac{K_2}{\delta}\right)\leq \Delta.
	\end{align}
\end{lemma}
\begin{proof}

	Observe that the equation $$\Delta\geq  K_1\delta\log\left(\frac{K_2}{\delta}\right)$$ is equivalent to 
	$$\frac{\Delta}{K_1K_2}\geq \frac{\delta}{K_2}\log\left(\frac{K_2}{\delta}\right),$$
	which is in turn equivalent to 
	$$\frac{K_1K_2}{\Delta}\leq \frac{\frac{K_2}{\delta}}{\log\left(\frac{K_2}{\delta}\right)}.$$

	By Lemma~\ref{lem:justcomputeit} this will certainly be satisfied as long as 
	\begin{align}
	\frac{K_2}{\delta}\geq 2\frac{K_1K_2}{\Delta} \log\left(\frac{K_1K_2}{\Delta}\right),
	\end{align}

\noindent	which can be equivalently rewritten
	
	\begin{align}
	\delta\leq \frac{\frac{\Delta}{2K_1}}{\log\left(\frac{K_1K_2}{\Delta}\right)},
	\end{align}

\noindent	as expected.

\end{proof}

\begin{lemma}
	\label{lem:justcomputeit}
	If $x:= 2y\log(y)$ then we have \begin{align}
	\frac{x}{\log(x)}\geq y.
	\end{align}
\end{lemma}

\begin{proof}
	We have 
	\begin{align}
	\frac{x}{\log(x)}&=\frac{2y\log(y)}{\log(y)+\log(2\log(y))}\\
	&\geq \frac{2y\log(y)}{\log(y)+\log(y)}\\
	&=y,
	\end{align}
	where at the second line we have used the inequality $\log(y)\leq y/2$.
\end{proof}

\section{Concentration results for exact recovery}

In this section we prove that the conditions of Lemma~\ref{lem:thekey} hold with high probability as long as the number of samples $N$ is large enough. The proofs here are very similar to the analogues in~\cite{IMCtheory1} (in some cases we quote the relevant result directly whenever this can be done without the need to modify the proof). However, some slight modifications are needed to make the results more general in the arbitrary tolerance thresholds $\delta_3$ etc., and also to remove the condition $N=|\Omega|\leq |\Omega_1|$ from the reference in question as we are interested in compact results that hold for any value of $N$ and in particular in the $N\rightarrow \infty$ limit.

\begin{lemma}[Adaptation of Lemma 5 in~\cite{IMCtheory1}]
	
	\label{lem:delta3}
	For any $\delta_3>0$ we have with probability $\geq 1-\delta_3$: 
	\begin{align}
	\left\| P_T-\frac{mn}{N}P_TP_{\Omega}P_T    \right\|\leq \min\left(        \sqrt{\frac{8}{3} \log\left(\frac{m+n}{\delta_3}\right)\frac{r\mu^2(a+b)}{N}   } ,   \>\> \frac{8}{3}\frac{r\mu^2(a+b)}{N}   \log\left(\frac{m+n}{\delta_3}\right)  \right).
	\end{align}

	In particular, as long as 
	\begin{align}
	N\geq T_3:=\frac{32}{3}r\mu^2(a+b)   \log\left(\frac{m+n}{\delta_3}\right),
	\end{align}
	we have  for any $Z\in\R^{m\times n}$: 
	\begin{align}
	\frac{mn}{N}\left\langle Z, P_TP_{\Omega}P_T(Z)\right\rangle \geq \frac{\|P_T(Z)\|_{\Fr}^2}{2}.
	\end{align}
	
\end{lemma}
\begin{proof}
	
	As computed in~\cite{IMCtheory1} we have the following values for the $\sum_{k=1}^N\rho_{k}^2$ and $M$ from Lemma~\ref{lem:bernsteinwithamax}: 
	\begin{align}
	\label{eq:Mrhofromxu1}
	M&:=\frac{r\mu^2(a+b)}{N},\\
	\sum_{k=1}^N\rho_{k}^2&=\frac{r\mu^2(a+b)}{N}.
	\end{align}

	Plugging these values into Lemma~\ref{lem:bernsteinwithamax} we immediately obtain: 
	\begin{align}
	\label{eq:directly1}
	\left\| P_T-\frac{mn}{N}P_TP_{\Omega}P_T    \right\|\leq \min\left(        \sqrt{\frac{8}{3} \log\left(\frac{m+n}{\delta_3}\right)\frac{r\mu^2(a+b)}{N}   } ,   \>\> \frac{8}{3}\frac{r\mu^2(a+b)}{N}   \log\left(\frac{m+n}{\delta_3}\right) \right)
	\end{align}
	as expected. 
	
	As for the second part of the theorem, not that if $N\geq T_3:= \frac{32}{3}r\mu^2(a+b)   \log\left(\frac{m+n}{\delta_3}\right)$, then we clearly have 
	\begin{align}
	\frac{8}{3} \log\left(\frac{m+n}{\delta_3}\right)\frac{r\mu^2(a+b)}{N} \leq \frac{1}{4},
	\end{align}
	
	and therefore by equation~\eqref{eq:directly1}, 
	
	\begin{align}
	\left\| P_T-\frac{mn}{N}P_TP_{\Omega}P_T    \right\|&\leq \min\left(        \sqrt{\frac{8}{3} \log\left(\frac{m+n}{\delta_3}\right)\frac{r\mu^2(a+b)}{N}   } ,   \>\> \frac{8}{3}\frac{r\mu^2(a+b)}{N}   \log\left(\frac{m+n}{\delta_3}\right) \right)\leq \frac{1}{2}.
	\end{align}
	This in turn implies that 
	\begin{align}
	\frac{1}{2}\|P_T(Z)\|_{\Fr}^2&\geq 
	\left\langle Z,P_T(Z)-  \frac{mn}{N}P_TP_{\Omega}P_T  (Z)     \right\rangle = 	\left\langle Z,P_T(Z)\right\rangle -\left\langle  Z,  \frac{mn}{N}P_TP_{\Omega}P_T  (Z)     \right\rangle, 
	\end{align}
	from which it follows that 
	
	\begin{align}
	\frac{1}{2}\|P_T(Z)\|_{\Fr}^2\leq  \frac{mn}{N}\left\langle  Z,  P_TP_{\Omega}P_T  (Z)     \right\rangle,
	\end{align}
	as expected. 
	
\end{proof}

\begin{lemma}[Substantially modified version of Lemma 6 in~\cite{IMCtheory1}]
	\label{lem:delta4}
	For any $\delta_4>0$ we have with probability $\geq 1-\delta_4$: 
	\begin{align}
	\left\|P_{T^{\top}}-\frac{mn}{N}P_{T^{\top}}P_{\Omega}P_{T^{\top}}\right\|\leq  \left( \sqrt{\frac{8}{3} \log\left(\frac{m+n}{\delta_4}\right)\frac{\mu^2(ab+r^2)}{N}   } ,   \>\> \frac{8}{3}\frac{\mu^2(ab+r^2)}{N} \log\left(\frac{m+n}{\delta_4}\right)    \right).
	\end{align}

	In particular, as long as 
	\begin{align}
	N\geq T_4:= \frac{32}{3}r\mu^2(a+b)   \log\left(\frac{m+n}{\delta_4}\right)
	\end{align}
	we have

	\begin{align}
	\left\langle Z,\frac{mn}{N}P_{T^{\top}}P_{\Omega}P_{T^{\top}}(Z)\right\rangle \leq     \frac{3}{2}\frac{a}{r} \|P_{T^{\top}}(Z)\|_{\Fr}^2.
	\end{align}

\end{lemma}

\begin{proof}

	By the calculation in~\cite{IMCtheory1} we can apply Lemma~\ref{lem:bernsteinwithamax} with the following values for $M, \rho$:

	\begin{align}
	\label{eq:Mrhofromxu2}
	M&:=\frac{\mu^2(ab+r^2)}{N}\\
	\sum_{k=1}^N\rho_{k}^2&=\frac{\mu^2(ab+r^2)}{N}.
	\end{align}
	
	From this, applying Lemma~\ref{lem:bernsteinwithamax} we immediately obtain:

	\begin{align}
	\label{eq:secondtime}
	\left\|P_{T^{\top}}-\frac{mn}{N}P_{T^{\top}}P_{\Omega}P_{T^{\top}}\right\|&\leq  \max \left( \sqrt{\frac{8}{3} \log\left(\frac{m+n}{\delta}\right)\sum_{k=1}^L \rho_k^2   } ,   \>\> \frac{8}{3}M   \log\left(\frac{m+n}{\delta}\right)    \right)\\
	&\leq \max \left( \sqrt{\frac{8}{3} \log\left(\frac{m+n}{\delta}\right)\frac{\mu^2(ab+r^2)}{N}   } ,   \>\> \frac{8}{3}\frac{\mu^2(ab+r^2)}{N} \log\left(\frac{m+n}{\delta}\right)    \right),
	\end{align}
	as expected.
	

	Note that 
	\begin{align}
	T_4&:= \frac{32}{3}r\mu^2(a+b)   \log\left(\frac{m+n}{\delta_4}\right)\\
	&\geq \frac{32}{3}\frac{r}{a}\mu^2(ab+r^2) \log\left(\frac{m+n}{\delta_4}\right).
	\end{align}
	(Indeed, $r^2/a\leq a$ (because $r\leq a$) and $rab/a=rb$.)

	Thus we certainly have 
	\begin{align}
	\left\|P_{T^{\top}}-\frac{mn}{N}P_{T^{\top}}P_{\Omega}P_{T^{\top}}\right\|
	&\leq \max \left( \sqrt{\frac{8}{3} \log\left(\frac{m+n}{\delta}\right)\frac{\mu^2(ab+r^2)}{N}   } ,   \>\> \frac{8}{3}\frac{\mu^2(ab+r^2)}{N} \log\left(\frac{m+n}{\delta}\right)    \right)\\
	&\leq \max\left( \sqrt{\frac{a}{4r}},\frac{a}{4r} \right)\leq \frac{a}{2r}.
	\end{align}

	This implies we have 
	\begin{align}
	\left\langle Z,\frac{mn}{N}P_{T^{\top}}P_{\Omega}P_{T^{\top}}(Z)\right\rangle \leq \left[1+ \frac{a}{2r}\right]\|P_{T^{\top}}(Z)\|_{\Fr}^2\leq     \frac{3a}{2r} \|P_{T^{\top}}(Z)\|_{\Fr}^2,
	\end{align}
	as expected.

\end{proof}

\begin{lemma}
	\label{lem:boundpt}
	For all $\delta_3,\delta_4,\delta_5>0$, for any $Z\in\R^{m\times n}$ such that $P_{\Omega}(Z)=0$ and $P_XZP_Y=Z$, we have with probability $\geq 1-\delta_3-\delta_4-\delta_5$: 
	\begin{align}
	\|P_T(Z)\|_{\Fr}\leq   \sqrt{3\tau_5} \sqrt{\frac{a}{r}}     \|P_{T^\top}(Z)\|_{\Fr},
	\end{align}
	as long as $N\geq \max(T_3,T_4)$ where $T_3,T_4$ are defined as in Lemmas~\ref{lem:delta3} and~\ref{lem:delta4}.

\end{lemma}

\begin{proof}
	First note that since $P_{\Omega}(Z)=0$ and $P_XZP_Y=Z$ we certainly have that 
	
	\begin{align}
	\left\langle P_{\Omega}P_{T}(Z),P_{\Omega}P_{T}(Z)\right\rangle =	\left\langle P_{\Omega}P_{T^\top}(Z),P_{\Omega}P_{T^\top }(Z)\right\rangle,
	\end{align}
	
\noindent	which implies 
	
	\begin{align}
	\label{eq:correctXu}
	\left\langle Z,P_{T}P_{\Omega}^2P_{T}(Z)\right\rangle =	\left\langle Z,P_{T^\top}P_{\Omega}^2P_{T^\top }(Z)\right\rangle.
	\end{align}
	
	Next, observe also that 
	\begin{align}
	\label{eq:notthatobvious}
	\left\langle Z,P_{T}P_{\Omega}P_{T}(Z)\right\rangle	&= 	\left\langle P_{T}Z,P_{\Omega}P_{T}(Z)\right\rangle	=\sum_{(i,j)} [P_{T}Z]_{i,j}^2h_{i,j}\nonumber \\&\leq \sum_{(i,j)} [P_{T}Z]_{i,j}^2h_{i,j}^2=\left\langle P_{T}Z,P_{\Omega}^2P_{T}(Z)\right\rangle
	 = \left\langle Z,P_{T}P_{\Omega}^2P_{T}(Z)\right\rangle,
	\end{align}
	where $h_{i,j}$ denotes the number of times that entry $(i,j)$ was sampled. 
	
	Now, by lemma~\ref{lem:delta5} we have that with probability $\geq 1-\delta_5$, $h_{i,j}\leq \tau_5$ for all $i,j$. Thus under the same condition we also have similarly to equation~\eqref{eq:notthatobvious}
	\begin{align}
	\label{eq:alsonotthatobvious}
	\left\langle Z,P_{T^\top}P_{\Omega}^2P_{T^\top}(Z)\right\rangle	\leq  \tau_5	\left\langle Z,P_{T^\top}P_{\Omega}P_{T^\top}(Z)\right\rangle.
	\end{align}

	Now by Lemmas~\ref{lem:delta3} and ~\ref{lem:delta4}	together with the above, we have with probability $\geq 1-\delta_3-\delta_4-\delta_5$: 
	\begin{align}
	\frac{1}{2}\|P_{T}(Z)\|_{\Fr}^2 &\leq \frac{mn}{N} 	\left\langle Z,P_{T}P_{\Omega}P_{T}(Z)\right\rangle     \label{eq:1line}\\
	&\leq      \frac{mn}{N} 	\left\langle Z,P_{T}P_{\Omega}^2P_{T}(Z)\right\rangle                 \label{eq:2line}\\
	&\leq           \frac{mn}{N} 	\left\langle Z,P_{T^\top}P_{\Omega}^2P_{T^\top }(Z)\right\rangle       \label{eq:3line}\\
	& \leq    \tau_5 \frac{mn}{N} 	\left\langle Z,P_{T^\top}P_{\Omega}P_{T^\top}(Z)\right\rangle    \label{eq:4line}\\
	&\leq \tau_5 \frac{3}{2} \frac{a}{r} \|P_{T^\top}(Z)\|_{\Fr}^2,      \label{eq:5line}
	\end{align}
	where at the first line~\eqref{eq:1line} we have used Lemma~\ref{lem:delta3};
	at the  second line~\eqref{eq:2line} we have used equation~\eqref{eq:notthatobvious};
	at the third line~\eqref{eq:3line} we have used equation~\eqref{eq:correctXu};  
	at the fourth line~\eqref{eq:4line} we have used equation~\eqref{eq:alsonotthatobvious}; and
	at the fifth and last line~\eqref{eq:5line} we have used Lemma~\ref{lem:delta4}. The result follows. 

\end{proof}

\begin{lemma}[Variation of Lemma 8 in~\cite{IMCtheory1}]
	\label{lem:delta6}
	Let $Z\in\R^{m\times n}$, for any $\delta_6>0$ as long as $N\geq T_6:=\frac{8}{3}\mu^2r(a+b) \log\left(\frac{m+n}{\delta_6}\right)  $  we have w.p. $\geq 1-\delta_6$: 
	\begin{align}
	\frac{mn}{N}\left\| P_{T^\top}P_{\Omega}P_{T}(Z)\right\| \leq \left\|P_{T}(Z)\right\|_{\infty} \sqrt{\frac{8\log\left (\frac{m+n}{\delta}\right)mn\mu\max(a,b)}{3N}}.
	\end{align}

\end{lemma}

\begin{proof}
	
	This again follows from an application of Bernstein's inequality~\ref{lem:bernsteinwithamax}. As calculated in~\cite{IMCtheory1} we note the following values for $M,\rho$: 
	
	\begin{align}
	M&=\left\|P_T(Z)\right\|_\infty \sqrt{\frac{mn\mu^2 (ab+r^2)}{N^2}}     \\
	\sum\rho_k^2&= \left\|P_T(Z)\right\|_\infty^2\frac{\mu\max(a,b)mn}{N}.
	\end{align}

	Thus Lemma~\ref{lem:bernsteinwithamax} immediately implies that with probability $\geq 1-\delta_6$: 
	\begin{align}
	\frac{mn}{N}\left\| P_{T^\top}P_{\Omega}P_{T}(Z)\right\| \leq \|P_{T}(Z)\|_{\infty}\max\left(\sqrt{\frac{8\mu mn\max(a,b)}{3N}\log\left(\frac{m+n}{\delta_6}\right)},\frac{8}{3}\sqrt{\frac{mn\mu^2(ab+r^2)}{N^2}}\log\left(\frac{m+n}{\delta_6}\right)   \right).
	\end{align}
	Thus as long as $N\geq \frac{8}{3}\mu  \frac{ab+r^2}{\max(a,b)}    \log\left(\frac{m+n}{\delta_6}\right)  $ we certainly have 
	
	\begin{align}
	\frac{mn}{N}\left\| P_{T^\top}P_{\Omega}P_{T}(Z)\right\| \leq \|P_{T}(Z)\|_{\infty}\sqrt{\frac{8\mu mn\max(a,b)}{3N}\log\left(\frac{m+n}{\delta_6}\right)},
	\end{align}
	as expected. The result follows upon noting that $T_6\geq \frac{8}{3}\mu  \frac{ab+r^2}{\max(a,b)}    \log\left(\frac{m+n}{\delta_6}\right) $.

\end{proof}

\begin{lemma}
	\label{lem:delta7}
	Let $Z\in\R^{m\times n}$ with probability $\geq 1-\delta_7$. Along as $N\geq T_7 \frac{8}{3}\mu^2r(a+b)   \log\left(\frac{2mn}{\delta_7}\right) $ we have 
	\begin{align}
	\left\| P_{T}(Z)-P_{T}P_{\Omega}P_{T}  (Z) \right\|_{\infty}\leq \|P_{T}(Z)\|_{\infty}\max  \sqrt{\frac{8}{3}\log\left(\frac{2mn}{\delta_7}\right)   \frac{\mu^2r(a+b)    }{N}  }.
	\end{align}
	In particular, as long as $N\geq T_7:=  \frac{32}{3}\mu^2r(a+b)   \log\left(\frac{2mn}{\delta_7}\right) $,  we have 
	\begin{align}
	\left\|  P_{T}(Z)-P_{T}P_{\Omega}P_{T}  (Z) \right\|_{\infty}\leq \frac{1}{2}\|P_{T}(Z)\|_{\infty}.
	\end{align}
	
\end{lemma}

\begin{proof}

	This follows from an application of the standard Bernstein inequality (i.e. Lemma~\ref{lem:bernsteinwithamax} with $m=n=1$) applied to each entry separately, together with a union bound over entries. As calculated in~\cite{IMCtheory1} we have the following values for "M" and "$\rho$": 
	\begin{align}
	M&=\frac{\mu^2r(a+b)  \|P_{T}(Z)\|_{\infty}   }{N}\\
	\sum \rho_k^2&=    \frac{\mu^2r(a+b)  \|P_{T}(Z)\|_{\infty}^2   }{N}.
	\end{align}
	
	Thus applying Lemma~\ref{lem:bernsteinwithamax} we see that for all $i,j$, the following holds with probability $\geq 1-\delta_7$:
	
	\begin{align}
	\left| \left[ P_{T}(Z)-P_{T}P_{\Omega}P_{T}  (Z) \right]_{i,j}\right|\leq \|P_{T}(Z)\|_{\infty}\max \left(  \sqrt{\frac{8}{3}\log\left(\frac{2}{\delta_7}\right)   \frac{\mu^2r(a+b)    }{N}  }   ,\frac{8}{3}\frac{\mu^2r(a+b)   \log\left(\frac{2}{\delta_7}\right)     }{N}\right),
	\end{align}
	and as long as $N\geq \frac{8}{3}\mu^2r(a+b)   \log\left(\frac{2}{\delta_7}\right) $ we have 
	\begin{align}
	\left| \left[ P_{T}(Z)-P_{T}P_{\Omega}P_{T}  (Z) \right]_{i,j}\right|\leq \|P_{T}(Z)\|_{\infty}\max  \sqrt{\frac{8}{3}\log\left(\frac{2}{\delta_7}\right)   \frac{\mu^2r(a+b)    }{N}  }.
	\end{align}
	Setting $\delta_7\leftarrow \delta_7/(mn)$ and taking a union bound over entries yields the first result immediately. The second result follows directly from the first.

\end{proof}

\section{Proof of Lemma~\ref{lem:asinxu}}
\label{sec:proofofxu}

First let us fix $\delta>0$. We will set $\delta_3=\delta_4=\delta_5=\delta_6=\delta_7$ for the lemmas above. 

Now, define 
\begin{align}
T= \frac{32}{3}\mu^2r(a+b)   \log\left(\frac{2mn}{\delta}\right).
\end{align}
Note that as long as $N\geq T$, we will have $N\geq \max(T_3,T_4,T_6,T_7)$ (with the same value $\delta$ used in all relevant theorems), which means the conditions of Lemmas~\ref{lem:delta3}, ~\ref{lem:delta4}, ~\ref{lem:delta5},~\ref{lem:boundpt}~\ref{lem:delta6}, and ~\ref{lem:delta6} are all satisfied. Indeed, it is trivially the case that $N\geq \max(T_3,T_4,T_6)$. As for $T_7$, the inequality $N\geq T_7$ still follows upon noticing that $\log(2mn)\leq \log((m+n)^2)=2\log(m+n)$. 

Following the ideas from~\cite{SimplerMC} and~\cite{IMCtheory1} we now construct a matrix $Y\in\R^{m\times n}$ with the properties from Lemma~\ref{lem:thekey}. 

We assume that $N\geq qT$ where $q$ will be determined later. We randomly select $q$ disjoint subsets of samples, each of size $T$, denoted by $\Omega_1,\Omega_2,\ldots,\Omega_q$, so we have 
\begin{align}
|\Omega_i|=T \quad  \quad \forall i\leq q.
\end{align}

As in Lemma~\ref{lem:thekey} we define $U$ to be the dual certificate of $R=XM^*Y^\top$ with respect to the norms $\|\nbull\|_{\spec}$ and $\|\nbull\|_{\nuc}$ and the standard Frobenius inner product. Note that because unlike~\cite{IMCtheory1} we do not assume that the columns of $X,Y$ are normed, we need to be a bit more careful about computing the relevant norms. By definition we certainly have $\|U\|_{\spec}=1$. However, to apply the above results, we will also need a bound on $\|U\|$. 

\begin{lemma}
	\label{lem:columnsarenotnormed}
	Let $U$ be the dual certificate of $R=XM^*Y^\top$ with respect to the norms $\|\nbull\|_{\spec}$ and $\|\nbull\|_{\nuc}$ and the standard Frobenius inner product. Let $\sigma_0$ denote the smallest singular value of $X$ and $Y$ (after preprocessing of $X,Y$ into matrices with orthogonal columns (ordered by decreasing norms), $\sigma_0$ denotes the minimum of the norm of the last column of $X$ and that of $Y$). We have the following bound on the ordinary spectral norm of $U$: 
	
	\begin{align}
	\|U\|\leq \sigma_0^{-2}.
	\end{align}
	
\end{lemma}

\begin{proof}

	By definition of $U$, $\|U\|_{\spec}=1$. By definition of the norm $\|\nbull\|_{\spec}$ we have $\|U\|_{\spec}=X^\top UY=\Sigma_1\widebar{X}^\top U\widebar{Y} \Sigma_2$ where $\widebar{X},\widebar{Y}$ are obtained by normalising the columns of $X,Y$ and $\Sigma_1,\Sigma_2$ are diagonal matrices containing the singular values of $X,Y$. We can now write

	\begin{align}
	\|U\|&=\|\widebar{X}^\top U\widebar{Y} \| =   \|\Sigma_1^{-1}[\Sigma_1\widebar{X}^\top U\widebar{Y} \Sigma_2]\Sigma_2^{-1}  \|   \leq \sigma_0^{-2}  \|\Sigma_1^{-1}[\Sigma_1\widebar{X}^\top U\widebar{Y} \Sigma_2]\Sigma_2^{-1}  \|   =\sigma_0^{-2} \|U\|_{\spec},    \sigma_0^2\|\widebar{X}^\top U\widebar{Y}\|= \sigma_0^2\|U\|,
	\end{align}
	as expected.

\end{proof}

Armed with the above, we continue the construction of our approximation $Y$ of $U$: 
we generate a sequence $Y_1,\ldots,Y_q$ as follows: 
\begin{align}
\mathcal{Y}_t=\frac{mn}{T}\sum_{i=1}^t P_{\Omega_i}(W_i)
\end{align}
where $W_1=U$ and $W_{t+1}$ is defined inductively as follows: 
\begin{align}
W_{t+1}&=P_{T}(U-\mathcal{Y}_{t})= W_t-\frac{mn}{T}P_TP_{\Omega_t}(W)\\
&= \left(P_T-\frac{mn}{T}P_TP_{\Omega_t}P_T\right)W_t.
\end{align}

Finally, we set $\mathcal{Y}=\mathcal{Y}_q$. 

\textbf{Remark:} the subsets $\Omega_i$ of the original sample are subsets of the observations rather than subsets of the entries. In particular, they can contain several obsevations of the same entry (and this is accounted for in the Lemmas above).

In the next two lemmas, we will now show that $Y$ satisfies the conditions of Lemma~\ref{lem:thekey} with high probability.

\begin{lemma}[Improved version of Lemma 10 in~\cite{IMCtheory1}]
	\label{lem:YapproachesU}
	Assume that $N\leq mn$. 
	
	With probability $\geq 1-5q\delta$, as long as $N\geq Tq$ and  \begin{align}
	q\geq  q_0&:=8\log(\sigma_0^{-1})+2\log(a)+4+\log(\tau)\\
	&=8\log(\sigma_0^{-1})+2\log(a)+4+\log\left(5\log\left(\frac{2mn}{\delta}\right)\right)
	\end{align}
	where $\tau=\tilde{\tau}_5=5\log(\frac{2mn}{\delta})$, 
	we have 
	\begin{align}
	\left\|P_T(\mathcal{Y})-U     \right\|_{\Fr}\leq\frac{1}{4}\sqrt{\frac{r}{3a\tau}}\frac{1}{\sigma_0^{-2}}.
	\end{align}

	Without the condition $N\geq mn$, the lemma still holds with $\tau=\tau_5=    \frac{N}{mn}+\frac{8}{3}\log\left(\frac{2mn}{\delta}\right) \sqrt{\frac{N}{mn}}$ (which depends logarithmically on $N$).

\end{lemma}
\begin{proof}

	Setting $\delta_3=\ldots=\delta_7=\delta$ in all Lemmas above we have that (as long as $N\geq qT$) all the high probability events of Lemmas Lemmas~\ref{lem:delta3}, ~\ref{lem:delta4}, ~\ref{lem:delta5},~\ref{lem:boundpt}~\ref{lem:delta6}, and ~\ref{lem:delta6} hold on \textit{each of the groups of samples $\Omega_i$} with probability $\geq 1-5q\delta$. 
	
	Since $W_{t+1}= \left(P_T-\frac{mn}{T}P_TP_{\Omega_t}P_T\right)W_t,$ and $\|W_1\|=\|U\|\leq \sigma_0^{-2}$ (by Lemma~\ref{lem:columnsarenotnormed}), we can apply Lemma~\ref{lem:delta3} iteratively and obtain: 
	\begin{align}
	\left\|W_{q+1}\|=\|P_{T}(\mathcal{Y})-U\right\|\leq \sigma_0^{-2} \prod_{i=1}^q \left\|     P_T-\frac{mn}{T} P_TP_{\Omega_i}P_T\right\|\leq  \frac{\sigma_0^{-2}}{2^q}.
	\end{align}

	Now, from this we obtain: 
	\begin{align}
\left \|P_{T}(\mathcal{Y})-U\right\|_{\Fr}\leq \frac{\sigma_0^{-2} \sqrt{\min(a,b)}  }{2^q}.
	\end{align}

	Thus, we see that the Lemma's statement will hold as long as we set
	\begin{align}
	q\geq q_0&:=8\log(\sigma_0^{-1})+2\log(a)+4+\log(\tau)\\&\geq  8\log(\sigma_0^{-1})+2\log(a)+\log(48)+\log(\tau)\geq \log_2\left[\sigma_0^{-4} \sqrt{\min(a,b)}4\sqrt{\frac{3a\tau}{r}}\right].
	\end{align}

\end{proof}

We will need the following additional Lemma. 

\begin{lemma}
	\label{lem:11notsoeasy}
	
	Let $U$ be the dual certificate of $R$, we have the following bound on the maximum entry of $U$: 
	\begin{align}
	\|U\|_{\infty}\leq \sqrt{\frac{ r \mu_1 }{mn}}\sigma_0^{-2}.
	\end{align}
\end{lemma}

\begin{proof}
	Let  $M^*=A\Sigma B^\top$ be the singular value decomposition of the ground truth core matrix $M^*$. By definition of $U$ and the relevant norms we have 
	\begin{align}
	X^\top UY=AB^\top.
	\end{align}
	
	It follows that $$\widebar{X}^\top U\widebar{Y} =\Sigma_1AB^\top \Sigma_2,$$
	where as usual, $\widebar{X},\widebar{Y}$ are obtained from $X,Y$ by normalizing the columns, and $\Sigma_1,\Sigma_2$ are diagonal matrices containing the singular values of $X,Y$. Next we have 
	\begin{align}
	U= \widebar{X}\Sigma_1 AB^\top \Sigma_2 \widebar{Y},
	\end{align}
	the result follows by the incoherence assumption.

\end{proof}

\begin{lemma}[Modification of Lemma 11 in~\cite{IMCtheory1}]
	\label{lem:maxnormsarenteasy}
	With probability $\geq 1-5q_0\delta$ as long as $N\geq \bar{T}q_0$ we have 
	\begin{align}
	\left\|P_{T^\top}(\mathcal{Y})\right\|\leq \frac{1}{2},
	\end{align}
	where $\bar{T}:=4\frac{\mu_1}{\mu}\sigma_0^{-4}T $.
	
\end{lemma}

\begin{proof}
	Similarly to Lemma~\ref{lem:YapproachesU} under the condition $N\geq \bar{T}q_0$ if we randomly pick $q$ groups of samples each of size $\bar{T}\geq T$ we have, with probability $\geq 1-5q_0\delta$, that all the high probability events of the previous lemmas hold for each of the sets of samples $\Omega_t$.

	Now by Lemma~\ref{lem:delta7} we have 
	\begin{align}
	\label{eq:theendthefirst}
	\|W_{t+1}\|_{\infty}\leq \left\|    \left( P_{T}- P_{T}P_{\Omega_t}   P_T\right)\right\|\leq \frac{1}{2}\|W_{t}\|_{\infty}.
	\end{align}

	Next we also have 
	\begin{align}
	\|P_{T^\top}(Y)\|&\leq \sum_{t=1}^q \frac{mn}{T}	\left\| (P_{T}-P_{T}P_{\Omega_t}P_{T} )(W_t) \right\|\\
	&\leq  \sqrt{\frac{8\log\left (\frac{m+n}{\delta}\right)mn\mu\max(a,b)}{3\bar{T}}} \sum_{t=1}^q \|W_t\|_{\infty}\\
	&\leq  \sqrt{\frac{8\log\left (\frac{m+n}{\delta}\right)mn\mu\max(a,b)}{3\bar{T}}} \sum_{t=1}^q \frac{\sigma_0^{-2}}{2^{t-1}}\sqrt{\frac{r  \mu_1}{mn}}\\
	&\leq 2   \sqrt{\frac{8\log\left (\frac{m+n}{\delta}\right) \mu \mu_1 r\max(a,b)}{3\bar{T}}}\sigma_0^{-2}\\
	&\leq \frac{1}{2},
	\end{align} 
	where at the second line we have used Lemma~\ref{lem:delta6}, at the third line we have used equation~\eqref{eq:theendthefirst} as well as Lemma~\ref{lem:11notsoeasy}.

\end{proof}


We can now finally prove Lemma~\ref{lem:asinxu}.
\begin{proof}[Proof of Lemma~\ref{lem:asinxu}]
	\textbf{Case 1 :  $N\leq 2\log(\frac{mn}{\delta})mn$. }
	
	In this case, the lemma follows  (even with probability $\geq  1-5q\delta$  ) immediately from Lemmas~\ref{lem:delta3}, ~\ref{lem:delta4},~\ref{lem:delta5},~\ref{lem:boundpt},~\ref{lem:delta6}, and~\ref{lem:delta7} upon noting that we then have: 
	\begin{align}
	\tau&= \frac{N}{mn}+\frac{8}{3}\log\left(\frac{2mn}{\delta}\right) \sqrt{\frac{N}{mn}}\\
	&\leq 2\log(\frac{mn}{\delta})+\frac{8}{3}\log\left(\frac{2mn}{\delta}\right)\sqrt{2\log\left(\frac{mn}{\delta}\right)}\\
	&\leq 5\log^{\frac{3}{2}}\left(\frac{mn}{\delta}\right),
	\end{align}
	and therefore
	\begin{align}
	q  &=  8\log(\sigma_0^{-1})+2\log(a)+4+\log(\tau)         \\
	&=  8\log(\sigma_0^{-1})+2\log(a)+4+\log(5\log^{\frac{3}{2}}\left(\frac{mn}{\delta}\right))         \\
	&=4+8\log(\sigma_0^{-1})+2\log(a)+\log\left[5\log^{\frac{3}{2}}\left(\frac{mn}{\delta}\right)\right] \\
	&\leq 6+8\log(\sigma_0^{-1})+2\log(a)+\log\left[\log^{\frac{3}{2}}\left(\frac{mn}{\delta}\right)\right]\\
	&\leq 6+8\log(\sigma_0^{-1})+2\log(a)+\log\left[\log\left(\frac{mn}{\delta}\right)\right]\\
	&=\log\left[e^6 \sigma_0^{-8} a \log\left(\frac{mn}{\delta}\right) \right].
	\end{align}

	\textbf{Case 2:  $N\geq 2\log(\frac{mn}{\delta})mn$.}

	In this case, by Lemma~\ref{lem:sooomanysamples} we have with probabilty $\geq 1-\delta$ that each entry was sampled at least once. Hence, the dual certificate $U$ itself is in the image of $P_{\Omega}$ and we can simply set $Y=U$. Note also that in this case  $\|P_{T^\top}(Y)\|=\|P_{T^\top}(U)\|=0\leq \frac{1}{2}$, and of course $\|P_T(Y)-U\|_{\Fr}=\|P_T(U)-U\|_{\Fr}=\|0\|_{\Fr}=0\leq \frac{1}{4}\sqrt{\frac{r}{3a\tau}}\frac{1}{\sigma_0^2}$.
\end{proof}


\bibliography{Bibliographomic}
\end{document}